\documentclass{article}

\usepackage{microtype}
\usepackage{graphicx}
\usepackage{booktabs} % for professional tables

\usepackage{hyperref}

\usepackage{amsmath, amsthm}
\usepackage{amsfonts}
\usepackage{comment}
\usepackage{bbm} % We need this for pretty Indicator \mathbbm{1}.
\usepackage{dsfont}
\usepackage[utf8]{inputenc}
\usepackage[T1]{fontenc}
\usepackage{amsmath}
\usepackage{amsfonts}
\usepackage{amssymb}
\usepackage[version=4]{mhchem}
\usepackage{stmaryrd}
\usepackage[export]{adjustbox}
\usepackage[accepted]{icml2025}

\usepackage{amsmath}
\usepackage{amssymb}
\usepackage{mathtools}
\usepackage{amsthm}

\usepackage[capitalize,noabbrev]{cleveref}

\usepackage{subcaption}

\theoremstyle{plain}
\newtheorem{theorem}{Theorem}[section]

\newtheorem{lemma}[theorem]{Lemma}
\newtheorem{corollary}[theorem]{Corollary}
\theoremstyle{definition}

\theoremstyle{remark}
\newtheorem{remark}[theorem]{Remark}

\usepackage[textsize=tiny]{todonotes}

\newcommand{\p}[2][]{{\mathbb{P}}_{#1} \left( \, #2 \, \right) } % функция для вероятности
\newcommand{\e}[2][]{{\mathbb{E}_{#1}\left[ \, #2 \,\right] }} % функция для МО
\newcommand{\var}[2][]{{\mathrm{Var}_{#1}\left[ \, #2 \,\right] }} % функция для дисперсии
\newcommand{\ra }{ \rightarrow} % просто стрелка
\newcommand{\asra}{ \overset{\textit{{a.s.}}}{\longrightarrow}} % сходимость почти наверное
\newcommand{\pra}{ \overset{\mathrm{p}}{\longrightarrow}} % сходимость по вероятности
\newcommand{\wra}{ \Rightarrow} % сходимость по распределению
\newcommand{\skk}[1]{\left\{ #1 \right\}} % фигурные скобки
\newcommand{\eps}{ \varepsilon}
\newcommand{\beq}{\begin{equation}}
\newcommand{\eeq}{\end{equation}}
\makeatletter
\def\namedlabel#1#2{\begingroup
    #2%
    \def\@currentlabel{#2}%
    \phantomsection\label{#1}\endgroup
}
\makeatother

\icmltitlerunning{Convergence, Sticking and Escape: Stochastic Dynamics Near Critical Points in SGD}

\begin{document}

\onecolumn
%\twocolumn[
\icmltitle{Convergence, Sticking and Escape: Stochastic Dynamics Near Critical Points in SGD}



\icmlsetsymbol{equal}{*}

\begin{icmlauthorlist}
\icmlauthor{Dmitry Dudukalov}{yyy}
\icmlauthor{Artem Logachov}{yyy}
\icmlauthor{Vladimir Lotov}{yyy}
\icmlauthor{Timofei Prasolov}{yyy}
\icmlauthor{Evgeny Prokopenko}{yyy}
\icmlauthor{Anton Tarasenko}{yyy}
\end{icmlauthorlist}

\icmlaffiliation{yyy}{Sobolev Institute of Mathematics,  Novosibirsk, Russia}

\icmlcorrespondingauthor{Dmitry Dudukalov}{d.v.dudukalov@math.nsc.ru}
\icmlcorrespondingauthor{Evgeny Prokopenko}{prokopenko@math.nsc.ru}

\icmlkeywords{SGD, sojourn time, nonconvex function, regularly varying}

\vskip 0.3in


\printAffiliationsAndNotice{} % otherwise use the standard text.

\begin{abstract}
    We study the convergence properties and escape dynamics of Stochastic Gradient Descent (SGD) in one-dimensional landscapes, separately considering infinite- and finite-variance noise. Our main focus is to identify the time scales on which SGD reliably moves from an initial point to the local minimum in the same ``basin''. Under suitable conditions on the noise distribution, we prove that SGD converges to the basin's minimum unless the initial point lies too close to a local maximum. In that near-maximum scenario, we show that SGD can linger for a long time in its neighborhood. 
    For initial points near a ``sharp'' maximum, we show that SGD does not remain stuck there, and we provide results to estimate the probability that it will reach each of the two neighboring minima. Overall, our findings present a nuanced view of SGD’s transitions between local maxima and minima, influenced by both noise characteristics and the underlying function geometry.
\end{abstract}

\section{Introduction}

%Many contemporary problems involve optimizing a target function over a high-dimensional parameter space, often requiring the use of massive datasets --- a challenge for which the practicality of many algorithms ranges from marginal to outright intractable. Stochastic Gradient Descent (SGD) and its variants, however, have demonstrated significant effectiveness under these conditions, owing to their capability to evaluate the cost function and its gradient using only a subset of training examples. This approach not only facilitates the processing of large datasets but also enables straightforward integration of new data in an ``online'' setting. In particular, the widespread use of SGD in neural network applications is motivated by the high computational cost of executing backpropagation across the entire training set.

The successful application of Stochastic Gradient Descent (SGD) to neural network training is often attributed to SGD's ability to avoid \textit{sharp local minima} and reach \textit{flat local minima} \cite{hochreiter1997flat, keskar2016large, li2018visualizing} within a reasonable timeframe, which improves performance of the resulting neural network on the test data. However, it was demonstrated that SGD with Gaussian noise exits the basin of any local minimum in exponentially long timescales \cite{freidlin1998random}. This has led to the belief that the empirical success of SGD is due to heavy-tailed noise. The presence of heavy tails has been shown in many deep learning tasks \cite{simsekli2019tail, csimcsekli2019heavy, nguyen2019non, mahoney2019traditional, garg2021proximal, srinivasan2021efficient}. \citet{simsekli2019tail} consider SGD as discrete approximations of L\'{e}vy driven Langevin equations and argue that the duration of time the SGD trajectory spends in each local minimum is proportional to the width of that minimum. Moreover, \citet{wang2021eliminating} analyze SGD with truncated heavy-tailed noise and show that under appropriate time scaling sharp local minima disappear from the trajectory, leaving only flat local minima.

Nevertheless, the initial location of SGD and the choice of time scale determine the system's convergence to various local statistical equilibria, a phenomenon known as metastability \cite{imkeller2008metastable}. Therefore, it is essential to select an appropriate amount of time for SGD to reach an effective solution to the optimization problem. The purpose of this paper is to rigorously study the two weaknesses of SGD: $(1)$ unsuitable time scaling and $(2)$ problematic starting points.

To outline our approach to the aforementioned problems and the resulting findings, we begin by formally defining the object of study. For a given loss function $f$ and a starting point $x \in \mathbb{R}$, SGD produces the following sequence:
\beq\label{sgdmc}
x_k^{\varepsilon}=x_{k-1}^{\varepsilon}-\varepsilon f'(x_{k-1}^{\varepsilon})+\varepsilon \xi_k, \ \ \ x_0^{\varepsilon}=x, \ \ \ k\in \mathbb{N},
\eeq
where $\varepsilon>0$ is the step-size of SGD (which eventually converges to zero) and $\xi_k$ represents random noise with zero mean.  We consider two different cases with regards to our noise (see conditions \ref{H1} and \ref{H2}): either the tail-distribution is regularly varying with parameter $\alpha \in (1, 2),$ which is clearly heavy-tailed,  or it has finite second moment, a scenario that includes both heavy and light tails. 
For a fixed $\varepsilon > 0$, equation~\eqref{sgdmc} defines a sequence in $n$ that either converges or oscillates, as illustrated in Figure~\ref{Himmelblau_function} for a two-dimensional case to provide visual intuition, although our analysis focuses on the one-dimensional setting. In this paper, we study the behavior of the number of iterations $n = n_{\varepsilon}$ as $\varepsilon \to 0$ in the framework of probabilistic limit theorems, under which key phenomena persist: convergence to a minimum, sticking to a critical point, and escaping from the neighborhood of a maximum. While each phenomenon is studied in a separate section, we highlight that their comparative analysis yields additional insights, and moreover, the proofs rely on a shared set of techniques.
\begin{figure}[H]
\includegraphics[width=17.3cm]{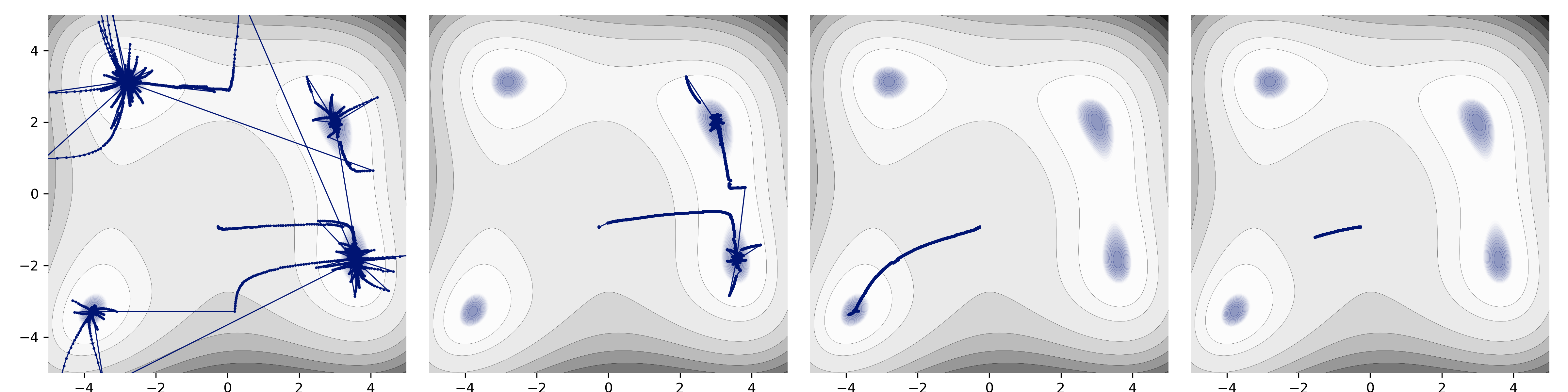}
\centering
\caption{Illustration of the behavior of the studied object. The Himmelblau function is considered as the objective function to be optimized. The noise $\xi_k$ is generated from an isotropic distribution, with the norm $\|\xi_k\|$ following a Pareto distribution with $\alpha=1.2$. We initialize the SGD at the point
$x_0^{\varepsilon} =(-0.27,-0.92)$, which is located in close proximity to a local maximum. The number of steps in the trajectory is the same across all four plots and equals $10^5$. The step size varies from left to right as follows: $\varepsilon=10^{-3},10^{-4},10^{-5},10^{-6}.$ }
\label{Himmelblau_function}
\end{figure}

Regarding the first phenomenon (i.e., convergence to a minimum) the condition that the number of steps is sufficiently small---characterized by $H(1/\varepsilon) n_\varepsilon \rightarrow 0$ in the regularly varying case (where $H$ denotes the tail function of the distribution of $|\xi_1|$) and by $\varepsilon^2 n_\varepsilon \rightarrow 0$ when the second moment is finite---ensures that the sequence remains within the initial basin of attraction.
However, if additionally $\varepsilon n_\varepsilon \rightarrow \infty$, then the sequence will have sufficient time to converge to a small neighborhood of the corresponding local minimum. We demonstrate that as $\varepsilon \rightarrow 0$, the outcome of SGD, halted at time $n_\varepsilon$, converges in probability to that local minimum. 
Furthermore, we establish that almost sure convergence holds provided that the growth rate of $n_{\varepsilon}$ is slightly reduced. Based on the analysis of our proofs, we conjecture that for $n_{\varepsilon}>\varepsilon^{-2}$, almost sure convergence does not hold, implying a restriction on the number of iterations to $n_{\varepsilon} \in (\varepsilon^{-1},\varepsilon^{-2})$. This constraint is well-known in the variable step-size setting \citep{RobbinsMonro1951}.

For the second phenomenon, if the starting point is positioned in a small neighborhood of a critical point which is not a local minimum, the convergence of the sequence to any local minima depends critically on the time scaling and the derivatives of the loss function. Hence, we first show conditions on the number of steps under which trajectory of SGD will not leave a neighborhood of the critical point. Crucial role  here is played by the number of zero derivatives of the loss function which informs how flat is this critical point.

In scenarios where the time scaling is sufficiently large to allow reaching local minima, an important question arises: which minimum will the sequence reach first? Given that the time scaling permits visiting only a single local minimum, only two minima are realistically achievable as the first point of convergence, possibly each with a non-zero probability. In general it is extremely difficult to answer this question analytically. In this paper we consider a simple case of a ``sharp'' maximum. More precisely, our loss function is a piecewise linear function in a neighborhood of such maximum. Here SGD will not ``stick'' to a critical point. The form of such loss function may be quite restrictive, nevertheless, such functions do appear in optimization literature \cite{leenaerts2013piecewise}. We consider the limiting probability of reaching the nearest local minimum on the right or left as $\varepsilon\to 0$
 and express it in terms of two accompanying random walks. We then use theory of random walks concerning reaching a boundary to give an upper bound for the latter formula.

In a broader context, our work fits within the literature studying the dynamics of metastable Markov chains. These are processes that, on small time scales, behave like a stable process defined on a subset of the original state space, while on larger time scales, they exhibit transitions between such subsets, rendering the internal dynamics within each subset effectively ``invisible''. For instance, several studies~\cite{betz2016multi, freidlin2017metastable} analyze the behaviour of metastable Markov chains on a finite state space, where transition probabilities depend on a small parameter~$\varepsilon$.
Specifically regarding SGD, \citet{wang2021eliminating} consider the case where the noise tail-distribution is regularly varying with parameter~$\alpha > 1$. They show that for sufficiently smooth functions~$f$ and a suitable starting point~$x$, the Markov chain~(\ref{sgdmc}) converges to a Markov process under the time scaling~$1/H(1/\varepsilon)$, provided that the clipping mechanism is neglected. The state space of this process consists solely of the local minima of~$f$, and the process visits each state, spending time proportional to the width of the minimum's basin.  However, a proof of weak convergence in the multidimensional case is highly technical \citep{wang2025global}, which obscures the crucial mathematical ideas behind the result. 
%
%The main objective of this paper is to provide a straightforward probabilistic proof of basic properties of SGD for non-convex functions. 
Unlike \citet{wang2021eliminating}, who focus on the regularly varying case, our results hold as long as either the noise has a finite second moment or the tail-distribution is regularly varying with parameter $\alpha \in (1, 2)$. Additionally, while they identify transitions between minima, we characterize two other distinct time scales: one where the chain explores the neighborhood of a maximum, and another where it explores the neighborhood of a minimum. We also explore SGD dynamics in underexplored scenarios, such as flat critical points and sharp maxima.

We consider a one-dimensional case with a function $f:\mathbb{R} \to \mathbb{R}$ that may possess a discontinuous derivative. While the $\mathbb{R}^1$ scenario may appear limited in practical applications, it provides a rigorous foundation for a comprehensive study of SGD. We assert that many of the observed effects can be effectively extended to higher dimensions. In support of this, \citet{wang2021eliminating} demonstrated that behaviours seen in the one-dimensional setting of SGD closely mirror those in multidimensional contexts. %Our methods should be particularly adaptable to situations where the initial starting point lies within an open, bounded region and the multi-dimensional target function is both convex and sufficiently regular.

%For more direct applications, our findings are pertinent to the \textit{coordinate descent algorithm}, which, in its simplest form, involves cyclic iteration through directions one at a time, optimizing the objective function with respect to each coordinate direction sequentially. This approach also proves advantageous in scenarios where gradient computation is impractical, albeit for reasons divergent from those previously mentioned. This connects closely with the \textit{line search algorithm}, which at each step selects a descent direction (potentially limited to one of the coordinate directions) and then performs a one-dimensional search to locate the minimum in that direction.

The rest of the paper is organized as follows. In Section~\ref{sec_suit_time_scal}, we present the conditions under which the SGD sequence converges (almost surely or in probability) to the minimum within the basin containing the initial value. 
In Section~\ref{sec_stic_crit_point}, we establish the conditions under which the SGD sequence remains in the neighborhood of a critical point for an extended period. 
Section~\ref{sec_sharp_max} addresses the problem of escaping from the neighborhood of a sharp maximum. 
Subsequently, Section~\ref{sec_lit_review} provides a review of related work and further discussion. 
Proofs and technical results are provided in Appendix~\ref{sec_appendix}.

\section{Main Results}
\subsection{Suitable Time Scaling}\label{sec_suit_time_scal}

In this section, we present results that quantify the number of iterations required for the SGD sequence to converge to a local minimum, depending on the assumptions imposed on the noise distribution. From the analysis of our proofs, a hypothesis emerges---known in the context of diminishing step sizes $(\varepsilon = \varepsilon_n)$ from the work~\citep{RobbinsMonro1951}---that for almost sure convergence one must take $n_{\varepsilon} \in \left( \varepsilon^{-1}, \varepsilon^{-2}\right)$ iterations. This hypothesis is confirmed by  simulations (see Figure~\ref{fig_sgd_convergence}). The observation is of practical importance for researchers employing diminishing step-size schedules with constant-step epochs (see, e.g., \cite{you2019does}), as almost sure convergence provides a crucial property:  it guarantees that SGD iterates converge as the training goes on. See Remark~\ref{H2rem1} for further discussion. 

Consider a continuous function $f \colon \mathbb{R} \to \mathbb{R}$ and assume that the starting point of SGD lies in the basin of attraction of a local minimum. Denote by $m \in \mathbb{R}$ the location of this local minimum and let $M_l < m < M_r$ be the boundaries of the basin, where $M_l$ (resp.\ $M_r$) may equal $-\infty$ (resp.\ $+\infty$). We assume that $f$ is differentiable on $(M_l, M_r) \setminus \{m\}$ with $f'(x) \neq 0$ for all $x \in (M_l, M_r) \setminus \{m\}$. Moreover, the derivative $f'$ is assumed to be bounded, right-continuous on $(M_l, M_r)$, and may exhibit jump discontinuities only at the point $m$.

Define the sequence
\begin{equation}\label{SGD_def_2}
x_k^{\varepsilon}=x_{k-1}^{\varepsilon}-\varepsilon f'(x_{k-1}^{\varepsilon})+\varepsilon \xi_k, \ \ \ x_0^{\varepsilon}=x, \ \ \ k\in \mathbb{N},
\end{equation}
where the random variables $\xi_1,\dots,\xi_k,\dots$ are independent and identically distributed  with zero mean $\e{\xi_1}=0$.
We will distinguish two cases (see conditions \ref{H1} and \ref{H2}), which may be interpreted as the cases of infinite and finite second moment of $\xi_1$.
%There is a slight gap between these two cases, and investigating it requires first establishing Zygmund’s Law of Large Numbers.
We assume that the noise $\xi_k$ is independent of the SGD iterate $x_{k-1}$. However, we believe that the results of this paper can be extended to the case of a martingale difference, i.e., when  $\e{\xi_k\,|\, \xi_1,\cdots,\xi_{k-1}} = 0$ (\citet{lauand2024}). In the latter case, both standard and typicality sampling-based mini-batch SGD (\citet{peng2019}) can be viewed as special cases of~\eqref{SGD_def_2}.

\subsubsection{Infinite Second Moment}
\begin{itemize}
\item[\namedlabel{H1}{$[\mathbf{H_1}]$}]  \qquad For all $u\geqslant 0$ and some $\alpha\in(1,2)$, the following equalities hold:
\[
    \overline{F}(u):=\p{\xi_1>u}=u^{-\alpha}L^+(u),
\]
\[
    F(-u):=\p{\xi_1\leq-u}=u^{-\alpha}L^-(u),
\]
where $L^+(u)$ and $L^-(u)$ are positive slowly varying functions at infinity (s.v.f.) such that
$$
\lim\limits_{u\rightarrow\infty}\frac{L^+(u)}{L^-(u)}=\kappa\in(0,\infty).
$$
\end{itemize}
Set $H(u):=\overline{F}(u)+F(-u)$. 
We need the class of functions $n_\varepsilon$, defined for $\varepsilon>0$ via growth rate as $\varepsilon\downarrow 0$:
\begin{equation*}
    \begin{split}
    \mathcal{N}_{\mathbf{H_1}} = \{ n_{\varepsilon} \geqslant 0 : n_\varepsilon \text{ is monotone and } \lim\limits_{\varepsilon\downarrow 0}\varepsilon \,n_\varepsilon=\infty, ~\lim\limits_{\varepsilon\downarrow 0}H\left(1/\varepsilon\right) n_\varepsilon = 0 \}.
    \end{split}
\end{equation*}
\begin{theorem} \label{t.1}
Suppose condition \ref{H1} is satisfied, and for some  $\Delta >0$, we have $x_0^{\varepsilon}~\in~(M_{l}~+~\Delta, M_{r}-\Delta).$
Then the following hold:
\begin{enumerate}
    \item[\namedlabel{1_t.1}{$(1)$}] For any function $n_\varepsilon \in \mathcal{N}_{\mathbf{H_1}}$ 
    $$
    x_{\lfloor n_\varepsilon \rfloor}^{\varepsilon} \pra m \text{ as } \varepsilon\downarrow 0.    
    $$
    Sign $\pra$ means convergence in probability.
    \item[\namedlabel{2_t.1}{$(2)$}] Let $L(u) \geqslant 0$ be a s.v.f. such that
    $ u L(u)$  is monotone increasing for $u\geqslant 0$ and $\e{|\xi_1|^\alpha L^\alpha(|\xi_1|)}<\infty$.
    Define $\overline{n}_\varepsilon := \left(\frac{1}{\varepsilon}\right)^\alpha L^\alpha\left(\frac{1}{\varepsilon}\right)$. Then $\overline{n}_\varepsilon \in \mathcal{N}_{\mathbf{H_1}}$ and for any $n_\varepsilon \in \mathcal{N}_{\mathbf{H_1}}$ satisfying $n_\varepsilon \leqslant \overline{n}_\varepsilon$ we have
        $$
    x_{\lfloor n_\varepsilon \rfloor}^{\varepsilon} \asra m  \text{ as } \varepsilon\downarrow 0.
    $$
    Sign $\asra$ means convergence almost surely.
\end{enumerate}
\end{theorem}

\begin{remark}\label{H1rem1}
\begin{enumerate}
    \item The larger the value of $\alpha$ (i.e., the lighter the tail of the noise distribution), the broader the class of sequences $\mathcal{N}_{\mathbf{H_1}}$ for which  convergence holds. In Theorem~\ref{t.1} we consider the case where the second moment is infinite; however, one may also examine the case $\alpha > 2$, corresponding to the existence of the second moment together with a regularity condition on the noise distribution. In this regime we conjecture that convergence in probability persists up to the level $H^{-1}(1/\varepsilon)$. However, we exclude this case for two reasons. First, our proof technique yields convergence only up to $L(1/\varepsilon) H(1/\varepsilon)$ for some s.v.f. $L(1/\varepsilon)$. Second, we conjecture that almost sure convergence fails to hold (see Remark~\ref{H2rem1}).
       \item  If we consider the timescale of order $H^{-1}\left( 1/\varepsilon \right)$, convergence to a minimum no longer occurs \cite{wang2021eliminating}. More precisely (see \citet[Corollary B.1.]{wang2021eliminating}), if we assume that $f$ has exactly $d\geqslant 1$ local minima at points $m_1,\dots,m_d$, and $d-1$ local maxima at points $M_1,\dots,M_{d-1}$ on its entire domain such that
$$
-\infty = M_0<m_1<M_1<\dots<m_d<M_d = \infty. $$
Then
$$
    x_{\lfloor t/H\left( 1/\varepsilon \right)\rfloor}^{\varepsilon} \wra Y_{t} \text{ as } \varepsilon \downarrow 0,
$$
where $\skk{Y_{t}}_{t \geqslant 0}$ is a Markov process wandering over the set of minima $\skk{m_1, \ldots, m_d}$ and sign $\wra$ means convergence in distribution. Thus, in the case of \ref{H1}, the condition $n_{\varepsilon} = o\left(H^{-1}\left( 1/\varepsilon \right)\right)$ is unimprovable for convergence to a minimum. Simulations illustrating the presence or absence of convergence are provided in the Appendix~\ref{sec_additional_sim_for_part_A}.

\item It follows from Theorem~\ref{t.1} that the upper bound in the set~\ref{H1} must be slightly reduced to ensure almost sure convergence. We conjecture that almost sure convergence fails to hold for $n_{\varepsilon} \geq \varepsilon^{-2}$ even when $\alpha > 2$. For further details, see the subsequent section addressing the case of noise with finite second moment.
\end{enumerate}
\end{remark}
\begin{figure}[htbp]
  \centering
  \begin{subfigure}[t]{0.45\textwidth}
    \includegraphics[width=\linewidth]{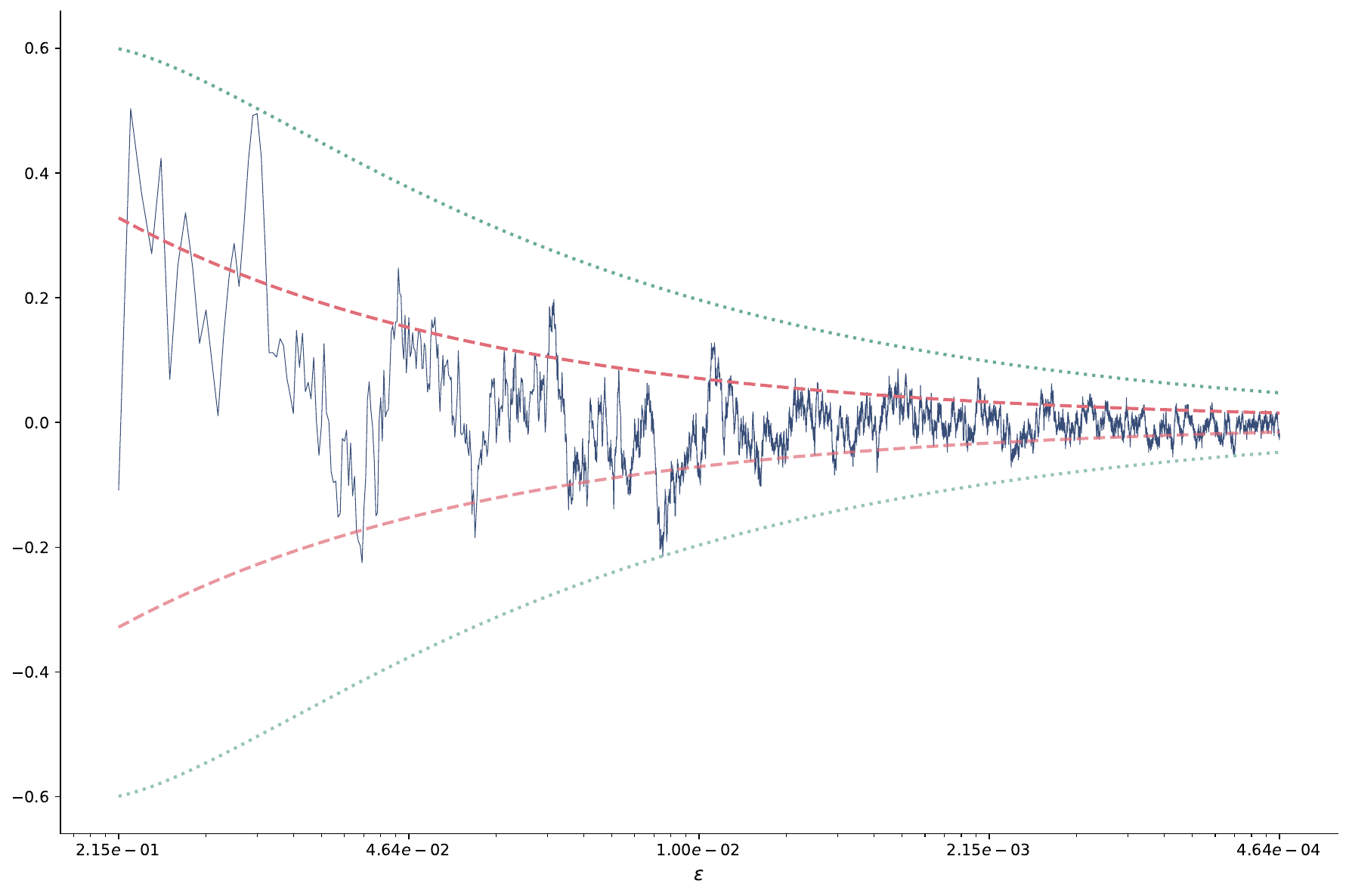}
    \caption{}
    \label{fig:sgd_1.5}
  \end{subfigure}
  \hspace{0.5cm}
  \begin{subfigure}[t]{0.45\textwidth}
    \includegraphics[width=\linewidth]{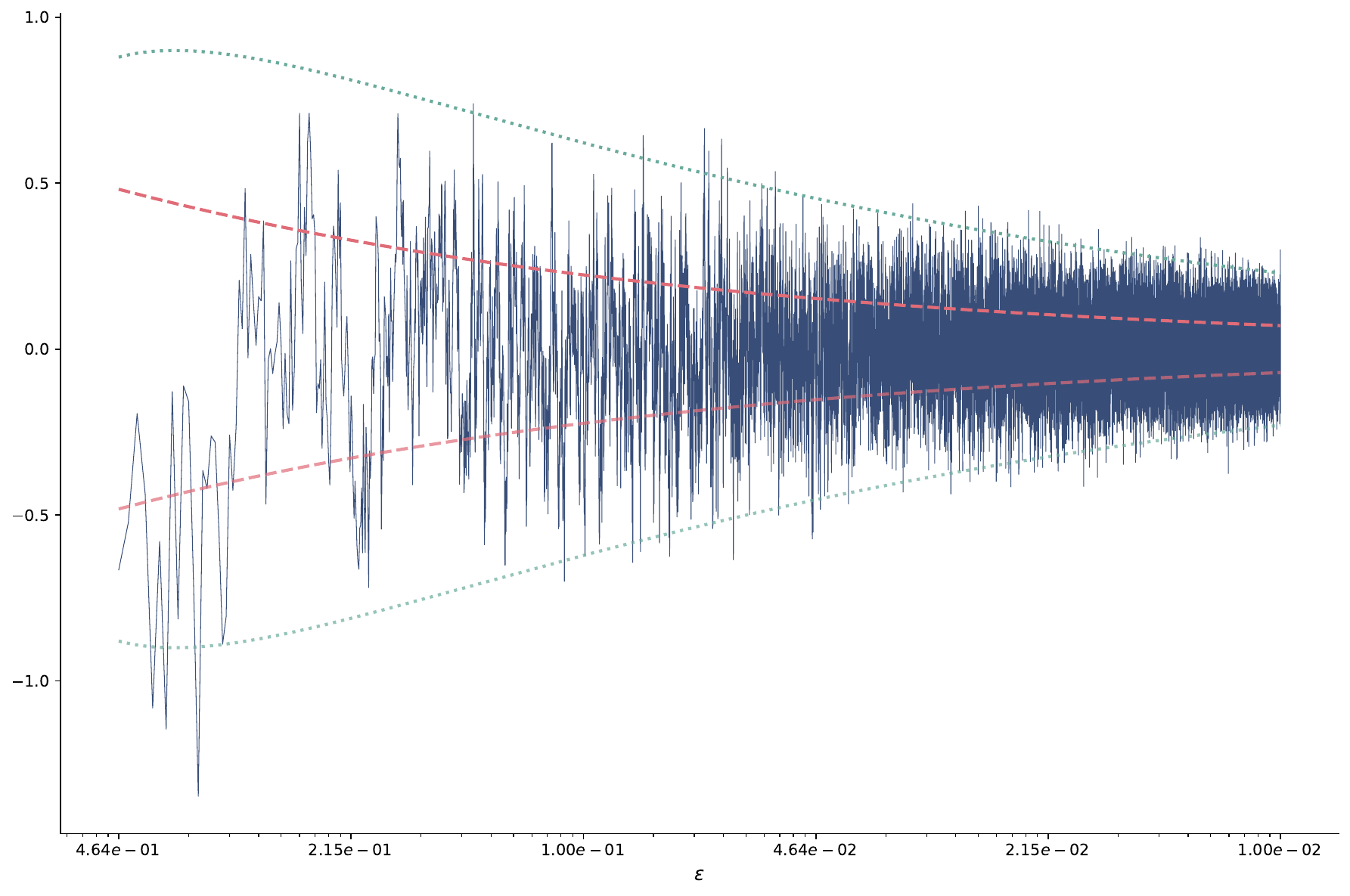}
    \caption{}
    \label{fig:sgd_3}
  \end{subfigure}
  \caption{Plot of $x_{n_{\varepsilon}}^{\varepsilon}$, where the SGD iteration is halted at $n_{\varepsilon}$ as $\varepsilon \downarrow 0$ (blue curve); the standard deviation $\sqrt{\operatorname{Var}(x_{n_{\varepsilon}}^{\varepsilon})}$ (red curve); and the <<law-of-the-iterated-logarithm>> bound $\sqrt{2\operatorname{Var}(x_{n_{\varepsilon}}^{\varepsilon}) \ln \ln n_{\varepsilon}}$ (green curve). 
The plots are generated for the quadratic objective $f(x) = x^{2}/2$ with additive noise $\xi \sim \mathcal{N}(0,1)$ under two asymptotic regimes: 
\cref{fig:sgd_1.5}~$n_{\varepsilon} = \varepsilon^{-3/2}$ and 
\cref{fig:sgd_3}~$n_{\varepsilon} = \varepsilon^{-3}$. 
One can easily observe that in regime~(b) almost sure convergence fails to hold, as the trajectory oscillates. Nevertheless, in both regimes the variance $\operatorname{Var}(x_{n_{\varepsilon}}^{\varepsilon})$ vanishes as $\varepsilon \to 0$, confirming that convergence in probability is preserved.}
  \label{fig_sgd_convergence}
\end{figure}

\subsubsection{Finite Second Moment}
\begin{itemize}
\item[\namedlabel{H2}{$[\mathbf{H_2}]$}]  \qquad  $\e{\xi_1^2}<\infty$.
\end{itemize}
We define the set $\mathcal{N}_{\mathbf{H_2}}$ as follows:
$$
\mathcal{N}_{\mathbf{H_2}} = \{ n_{\varepsilon} \geqslant 0 :~ \lim\limits_{\varepsilon\downarrow 0}\varepsilon\,n_\varepsilon=\infty,~ \lim\limits_{\varepsilon\downarrow 0}\varepsilon^2 n_\varepsilon = 0 \}.
$$
Thus, we do not require regular variation of the noise distribution's tails, but merely impose conditions on their rate of decay. The following theorem holds.
\begin{theorem} \label{t.2}
Suppose condition \ref{H2} is satisfied, and for some $\Delta >0$, we have $x_0^{\varepsilon}~\in~(M_{l}~+~\Delta, M_{r}-\Delta).$
Then the following hold:
\begin{enumerate}
    \item[\namedlabel{1_t.2}{$(1)$}] For any monotone function $n_\varepsilon \in \mathcal{N}_{\mathbf{H_2}}$,
    $$
    x_{\lfloor n_\varepsilon \rfloor}^{\varepsilon} \pra m \text{ as } \varepsilon\downarrow 0.    
    $$
   \iffalse
    \item[\namedlabel{2_t.2}{$(2)$}] \textcolor{blue}{Let $L(u) \geqslant 0$ be a s.v.f. such that
    $
    L(u) = o\left( (\ln \ln u)^{-\frac{1}{2}}\right).
    $
    Define $\overline{n}_\varepsilon := \left(\frac{1}{\varepsilon}\right)^2 L^2\left(\frac{1}{\varepsilon}\right)$. Then  $\overline{n}_\varepsilon \in \mathcal{N}_{\mathbf{H_2}}$ and for any $n_\varepsilon \in \mathcal{N}_{\mathbf{H_2}}$ satisfying $n_\varepsilon \leqslant \overline{n}_\varepsilon$ we have
        $$
    x_{\lfloor n_\varepsilon \rfloor}^{\varepsilon} \asra m \text{ as } \varepsilon\downarrow 0.
    $$
    }
    \fi
    
 \item[\namedlabel{2_t.2}{$(2)$}] Define $\overline{n}_\varepsilon := \left(\frac{1}{\varepsilon}\right)^2 \left(\ln \ln \left(\frac{1}{\varepsilon}\right)\right)^{-1}$. Then  for any $n_\varepsilon \in \mathcal{N}_{\mathbf{H_2}}$ satisfying $n_\varepsilon  / \overline{n}_\varepsilon \ra 0$ we have
        $$
    x_{\lfloor n_\varepsilon \rfloor}^{\varepsilon} \asra m \text{ as } \varepsilon\downarrow 0.
    $$
    
\end{enumerate}
\end{theorem}
\begin{remark}\label{H2rem1}
\begin{enumerate}
    \item Analogously to Theorem~\ref{t.1}, the upper bound in the set $\mathcal{N}_{\mathbf{H_2}}$ must be slightly reduced to guarantee almost sure convergence in Theorem~\ref{t.2}.
    \item The proof of almost sure convergence relies on the law of the iterated logarithm, which implies that almost sure convergence fails for properly normalized partial sums due to the discrepancy between their limit superior and limit inferior. Consequently, we conjecture that almost sure convergence also fails for $n_{\varepsilon} > \varepsilon^{-2}$. Although we do not have a rigorous proof of this claim, the absence of convergence is evident in numerical simulations (see Figure~\ref{fig_sgd_convergence}).
    \item Almost sure convergence possesses a crucial practical property: it guarantees that SGD iterates converge as the training goes on. This property is essential for practitioners employing diminishing step-size schedules with constant-step epochs (see, e.g., \cite{you2019does}). Consequently, our theorem establishes that if the number of iterations within each epoch smaller the critical threshold $\left(\frac{1}{\varepsilon}\right)^2\left(\ln \ln \left(\frac{1}{\varepsilon}\right)\right)^{-1}$, then SGD iterates converge. Conversely, we conjecture that exceeding this threshold leads to a breakdown of convergence.    It is interesting to compare this observation with the classical Robbins--Monro procedure for decaying step size, i.e., when $\varepsilon = \varepsilon_n$ \citep{RobbinsMonro1951}. Roughly speaking, the convergence condition in that framework requires $\varepsilon_n$ to satisfy $\frac{1}{n} < \varepsilon_n < \frac{1}{\sqrt{n}}$. Thus, we have derived an analogous condition, but for the constant step-size case.
    \end{enumerate}
\end{remark}

Throughout the following we will use the following notations
$\boldsymbol{x}:=(x^{(1)},\dots,x^{(d)})\in\mathbb{R}^d$
 and 
$||\boldsymbol{x}||:=\sqrt{(x^{(1)})^2
+\dots+(x^{(d)})^2}$, where integer $d\geqslant 2$.

\begin{remark}
   Let the random vectors $\boldsymbol{\xi}_k=(\xi_k^{(1)},\dots,\xi_k^{(d)})\in \mathbb{R}^d$, $k\in \mathbb{N}$, be independent and identically distributed, $\mathbb{E}\xi_k^{(r)}=0$, $1\leq r \leq d$, and assume that either the random variables $\xi_k^{(r)}$ satisfy condition \ref{H1} for all $1\leq r \leq d$, or $||\boldsymbol{\xi}_k||$ satisfies condition \ref{H2}.

We also require that initial point
$$
\boldsymbol{x}\in \boldsymbol{B}(\boldsymbol{m}, \tilde{R}):=\{\boldsymbol{y}\in\mathbb{R}^d:||\boldsymbol{y}-\boldsymbol{m}||<\tilde{R}\}
$$
for some $0<\tilde{R}<R$  
and function $f$  is strongly convex
in $\boldsymbol{B}(\boldsymbol{m},R)$.
Then for the sequence
$$
\boldsymbol{x}_k^{\varepsilon}=\boldsymbol{x}_{k-1}^{\varepsilon}-\varepsilon \nabla f(\boldsymbol{x}_{k-1}^{\varepsilon})+\varepsilon \boldsymbol{\xi}_k, \ \ \ \boldsymbol{x}_0^{\varepsilon}=\boldsymbol{x}, \ \ \ k\in \mathbb{N},
$$
the statements of Theorems \ref{t.1},
\ref{t.2} hold.

A rigorous proof of this fact seems to be obtainable based on the results of Chapter 2 of the monograph \citet{nesterov2004introductory}, concerning the properties of strongly convex functions, together with the derivations from the proofs of Lemmas \ref{l.1}, \ref{l.3}, which must be carried out for each coordinate of the vector
$\boldsymbol{x}_k^{\varepsilon}$.
\end{remark}

%%%%%%%%%%%%%%%%%%%%%%%%%%%%%%%%
%%%%%%%%%%%%%%%%%%%%%%%%%%%%%%%%
%%%%%%%%%%%%%%%%%%%%%%%%%%%%%%%%
\subsection{Sticking to a Critical Point}\label{sec_stic_crit_point}
%%%%%%%%%%%%%%%%%%%%%%%%%%%%%%%%
%%%%%%%%%%%%%%%%%%%%%%%%%%%%%%%%
%%%%%%%%%%%%%%%%%%%%%%%%%%%%%%%%

In this section, we establish conditions under which the SGD sequence remains in the vicinity of a critical point $c \in \mathbb{R}$ for an extended period. The point $c$ can be either an extremum or a inflection point.

Given the nature of this problem, we only require assumptions about the function to hold within a neighborhood of $c \in \mathbb{R}$. Specifically, suppose there is some $K \geqslant 1$ such that $c \in \mathbb{R}$ is a $K$-critical point of $f$. In other words, for all $k = 1, \ldots, K$ we have $f^{(k)}(c) = 0$ and also $f^{(K+1)}(c) \neq 0$. Furthermore, suppose there is a $\Delta > 0$ so that
\[
\sup_{c - \Delta \leqslant x \leqslant c + \Delta} |f^{(K+1)}(x)| < \infty.
\]

The statements in this section are rather technical; therefore, their precise formulations are provided in the Appendix~\ref{sec_ap_4}. Here, we shall instead present our results by introducing an asymptotic equality up to a slowly varying function:
\begin{equation*}
\delta_1(\varepsilon) \overset{\mathcal{L}}{=} \delta_2 (\varepsilon) \text{ iff there is a s.v.f. $L(\varepsilon)$ such that  } \delta_1(\varepsilon) = L(\varepsilon) \delta_2(\varepsilon).
\end{equation*}

Under these notation, the SGD sequence remains within a shrinking neighborhood of $c$. More precisely, the following two theorems apply.
\begin{theorem}\label{th1_crit_point_simple}
    Let \ref{H1} and some technical conditions hold. Then for all $t >0$  
     as $\varepsilon\downarrow 0$,
    \begin{equation*}
        \sup\limits_{n \leqslant t h(\varepsilon)}|x^{\varepsilon}_{n} - c| \leqslant \delta(\varepsilon) \text{ almost surely (a.s.) }
    \end{equation*}
    holds uniformly over all $x_0 \,:\, |x_0 - c|< \frac{1}{3} \delta(\varepsilon),$ where 
\begin{equation*}
        \delta(\varepsilon) \overset{\mathcal{L}}{=}          \varepsilon^{\frac{\alpha-1}{K-1+\alpha}} \ra 0,     \       h(\varepsilon) \overset{\mathcal{L}}{=}  \varepsilon^{-\frac{\alpha K}{K-1+\alpha}} \ra \infty. 
    \end{equation*}
\end{theorem}

\begin{theorem}\label{th2_crit_point_simple}
    Let \ref{H2} and some technical conditions hold.   
    Then, for all $t >0$ and as $\varepsilon\downarrow 0$ hold
    \begin{equation*}
        \sup\limits_{n \leqslant t h(\varepsilon)}|x^{\varepsilon}_{n} - c| \leqslant \delta(\varepsilon) \text{ a.s. }
    \end{equation*}
    holds uniformly over all $x_0 \,:\, |x_0 - c|< \frac{1}{3} \delta(\varepsilon),$ where
          \begin{equation*}
        \delta(\varepsilon) \overset{\mathcal{L}}{=}
        \varepsilon^{\frac{1}{K+1}} \ra 0, \     h(\varepsilon)  \overset{\mathcal{L}}{=}  \varepsilon^{-\frac{2K}{K+1}} \ra \infty.  
        \end{equation*}
\end{theorem}

\begin{remark}\label{rem_crit_point}
\begin{enumerate}
\item (convergence to maximum). 
From the statement of  Theorems~\ref{th1_crit_point_simple} and \ref{th2_crit_point_simple}, it follows that \(\delta(\varepsilon) \to 0\) and \(h(\varepsilon) \geq \varepsilon^{-1}\) for all sufficiently small \(\varepsilon > 0\). Consequently, even after a sufficiently large number of iterations — namely, \(h(\varepsilon)\) — the iterates will necessarily converge to a critical point, which may be a local maximum. This result implies that, in general, we are not guaranteed to descend to any minimum from an arbitrary starting point within a reasonably long time.

\item (simple example $f(x) = -(x-c)^2/2$). It is possible to describe full dynamic in the simple case $f(x) = -\frac{(x-c)^2}{2}$. Let  $x_0^{\varepsilon} = c+ \varepsilon,$ then possible to write the SGD sequence explicitly:
\begin{equation*}
x_n^{\varepsilon}  = \varepsilon(1+2\varepsilon)^n +  \varepsilon(1+\varepsilon)^n\sum_{r=1}^n\xi_r(1+\varepsilon)^{-r}.
\end{equation*}
If we now consider, in the case~\ref{H2} for example, the subsequence \( x_{t n_{\varepsilon}} \) for  a parameter \( t \geq 0 \) and  \(n_{\varepsilon} =  \frac{1}{2\varepsilon} \ln \frac{1}{\varepsilon}\) , then three distinct asymptotic regimes emerge as \( \varepsilon \downarrow 0 \):
\begin{enumerate}
    \item \( x_{t n_{\varepsilon}} \pra c \) for \( t \in [0, 1) \), i.e. the trajectory is sticked to the maximum;
    
    \item \( x_{t n_{\varepsilon}} \wra  \eta \), where \( \eta \sim \mathcal{N}\!\left(c, \frac{ \var{\xi_1}}{2}\right) \), for \( t = 1 \), i.e. the process balances near the maximum;
    
    \item \( \left|x_{t n_{\varepsilon}}\right| \pra \infty \) for \( t > 1 \),   i.e. the process  escapes the maximum.
\end{enumerate}
Thus, the threshold $n_{\varepsilon} = \frac{1}{2\varepsilon} \ln \frac{1}{\varepsilon}$ represents a sharp phase transition: for iteration counts asymptotically smaller than $n_{\varepsilon}$, the SGD iterates converge to a local maximum, whereas for asymptotically larger iteration counts, they fail to converge to the maximum. Since $K=1$ in this example, Theorem~\ref{th2_crit_point_simple} yields $h(\varepsilon)  \overset{\mathcal{L}}{=}   \varepsilon^{-1}  \overset{\mathcal{L}}{=}   \frac{1}{2\varepsilon} \ln \frac{1}{\varepsilon}$, thereby establishing the sharpness (optimality) of the bound in Theorem~\ref{th2_crit_point_simple} up to a s.v.f.

\item The number of iterations $h(\varepsilon)$ during which SGD remains in the neighborhood of $c$ increases with the number of vanishing derivatives $K$ of the function $f(x)$ at the point $c$. As $K \to \infty$, the number of iterations $h(\varepsilon)$ asymptotically (as $\varepsilon \to 0$) matches the upper bound for the order of $n_{\varepsilon}$ from Theorems~\ref{t.1} and \ref{t.2}.  

\item (pass through a inflection point).  There exists an entirely different perspective on Theorems~\ref{th1_crit_point_simple} and \ref{th2_crit_point_simple}, suggested by a referee, for which we presently lack a complete proof, yet the underlying idea is clearly visible. Suppose that a point \( c \) is a \( K \)-critical inflection point—for simplisity, $K$ is even and  we have 
$$
f^{\prime}(x)=(x-c)^K+o\left((x-c)^K\right) \text{ as } x \ra c.
$$
 One may then ask: what is the asymptotic time required for stochastic gradient descent (SGD) to pass through a neighborhood of such a point? The results of Theorems~\ref{th1_crit_point_simple} and \ref{th2_crit_point_simple} help address this question as follows. Let us now consider the case \ref{H1}, i.e. up to a s.v.f.
 \begin{equation}\label{eq_for_hdelta}
\frac{\varepsilon h^{1 / \alpha}(\varepsilon)}{\delta(\varepsilon)}  \overset{\mathcal{L}}{=}  1, \quad \frac{\varepsilon h(\varepsilon)}{\delta^{1-K}(\varepsilon)}  \overset{\mathcal{L}}{=}  1.
\end{equation}
Define the sequence
\begin{equation}\label{def_normalized_seq}
\hat{x}_t^\varepsilon \stackrel{\text { def }}{=} \frac{x_{\lfloor t \cdot h(\varepsilon)\rfloor}^\varepsilon-c}{\delta(\varepsilon)}.
\end{equation}
If $|x_0 - c| < \delta(\varepsilon), $ then from Theorem~\ref{th1_crit_point_simple}
$$
\begin{aligned}
\hat{x}_{k+1}^\varepsilon-\hat{x}_k^\varepsilon & = \frac{\varepsilon}{\delta(\varepsilon)} \sum_{i=k h(\varepsilon)+1}^{(k+1) h(\varepsilon)}\left(\xi_i- f^{\prime}\left(x_{i}^{\varepsilon}\right) \right) =  \frac{\varepsilon}{\delta(\varepsilon)} \sum_{i=k h(\varepsilon)+1}^{(k+1) h(\varepsilon)}\left(\xi_i- f^{\prime}\left(x_{kh(\varepsilon)}^{\varepsilon}\right)  + O\left( |x_{kh(\varepsilon)}^{\varepsilon} - x_{i}^{\varepsilon}|\right)\right)  \\ 
& = \frac{\varepsilon}{\delta(\varepsilon)} \sum_{i=k h(\varepsilon)+1}^{(k+1) h(\varepsilon)}\left(\xi_i-\left(x_{k h(\varepsilon)}^{\varepsilon}-c\right)^K+o\left(\left(x_{k h(\varepsilon)}^{\varepsilon}-c\right)^K\right) + O(\delta(\varepsilon))  \right) \\
& \stackrel{d}{=} \frac{\varepsilon h^{1 / \alpha}(\varepsilon)}{\delta(\varepsilon)}\left(\frac{1}{h^{1 / \alpha}(\varepsilon)} \sum_{i=1}^{h(\varepsilon)} \xi_i\right)-\frac{\varepsilon h(\varepsilon)}{\delta^{1-K}(\varepsilon)}\left(\frac{x_{k h(\varepsilon)}^{\varepsilon}-c}{\delta(\varepsilon)}\right)^K+o\left(\frac{\varepsilon h(\varepsilon)}{\delta^{1-K}(\varepsilon)}  \right) + O(\varepsilon\,h(\varepsilon)) \\
& \sim-\left(\hat{x}_k^{\varepsilon}\right)^K+\left(\frac{1}{h^{1 / \alpha}(\varepsilon)} \sum_{i=1}^{h(\varepsilon)} \xi_i\right).
\end{aligned}
$$
From generalized central limit theorem, we obtain a recursive equation that no longer contains $\varepsilon$
\begin{equation}\label{eq_rec_seq_limit}
\hat{x}_{k+1}^\varepsilon-\hat{x}_k^\varepsilon=-\left(\hat{x}_k^\varepsilon\right)^K+L(k+1)-L(k),
\end{equation}
where $L(t)$ is a Lévy process with $\alpha-$stable distribution. See Figure~\ref{fig_hist_one_step} for an illustration of the existence of the limiting distribution under the normalizations specified in equation~\eqref{def_normalized_seq}.

Thus, the time required to pass through a fixed neighborhood of the inflection critical point—say, \((c - \Delta, c + \Delta)\)—is asymptotically equivalent to the time needed for the sequence defined by equation~\eqref{eq_rec_seq_limit} to  pass through the expanding neighborhood \(\bigl(-\frac{\Delta}{\delta(\varepsilon)}, \frac{\Delta}{\delta(\varepsilon)}\bigr)\), multiplied by \(h(\varepsilon)\).  

Such a statement is possible only for normalizations satisfying condition~\eqref{eq_for_hdelta}, which precisely arise in our Theorem~\ref{th1_crit_point_simple}. This once again demonstrates the optimality of the bounds established in our results.

\item     The results of Section~ \ref{sec_stic_crit_point} are possible to generalize in the multidimensional case, but the derivations in this scenario become noticeably more complicated.
The condition \ref{H1}  is replaced by multivariate regularly varying (see chapter 6 in \citet{Resnick2007}) and \ref{H2} are replaced by \ref{H2} for the norm $\|\boldsymbol{\xi}_1\|.$ 
The proof of the results in Section~\ref{sec_stic_crit_point} can be reproduced almost verbatim, except that we use multidimensional Taylor's formula and replace the absolute value with the norm in \eqref{eq_crit_esc_1}.
\end{enumerate}
\end{remark}

\begin{figure}[H]
    \centering
    \begin{subfigure}[b]{0.24\textwidth}
        \centering
        \includegraphics[width=\linewidth]{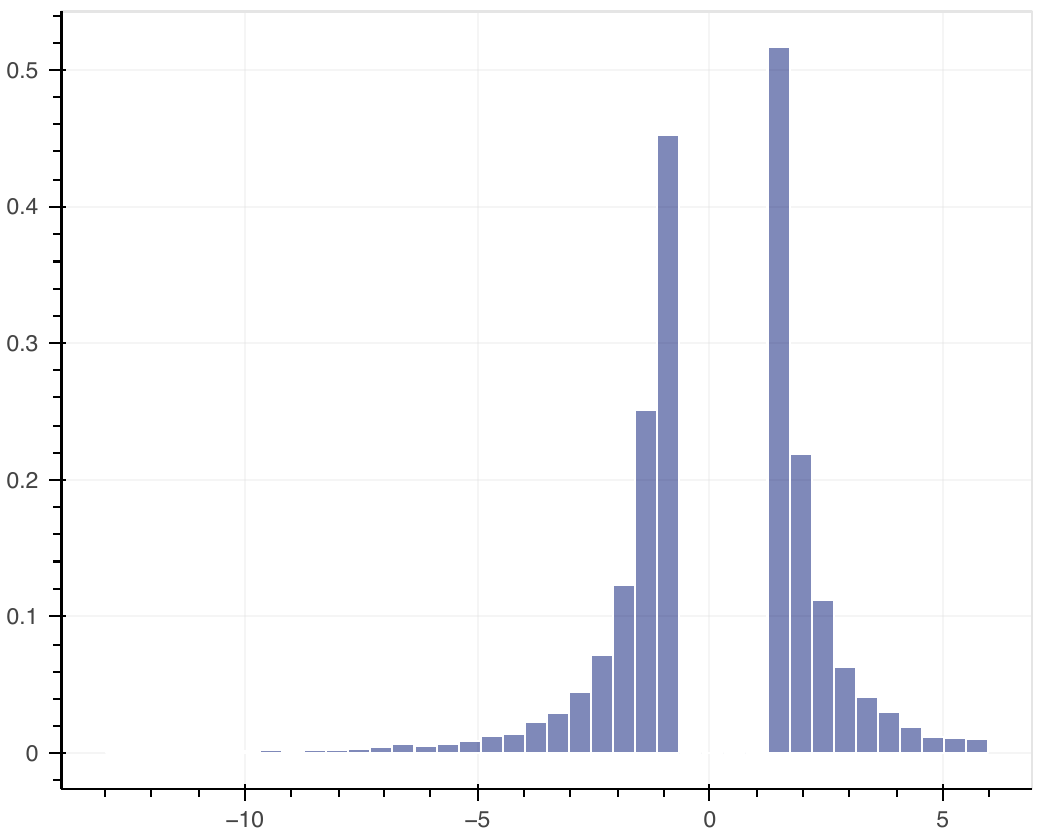}
        \captionsetup{justification=centering}
        \caption{ $\varepsilon = 1, \lfloor h \left( \varepsilon \right) \rfloor = 1,  \delta \left( \varepsilon \right) = 1$}
        \label{fig:hist_eps_1}
    \end{subfigure}
    \hfill
    \begin{subfigure}[b]{0.24\textwidth}
        \centering
        \includegraphics[width=\linewidth]{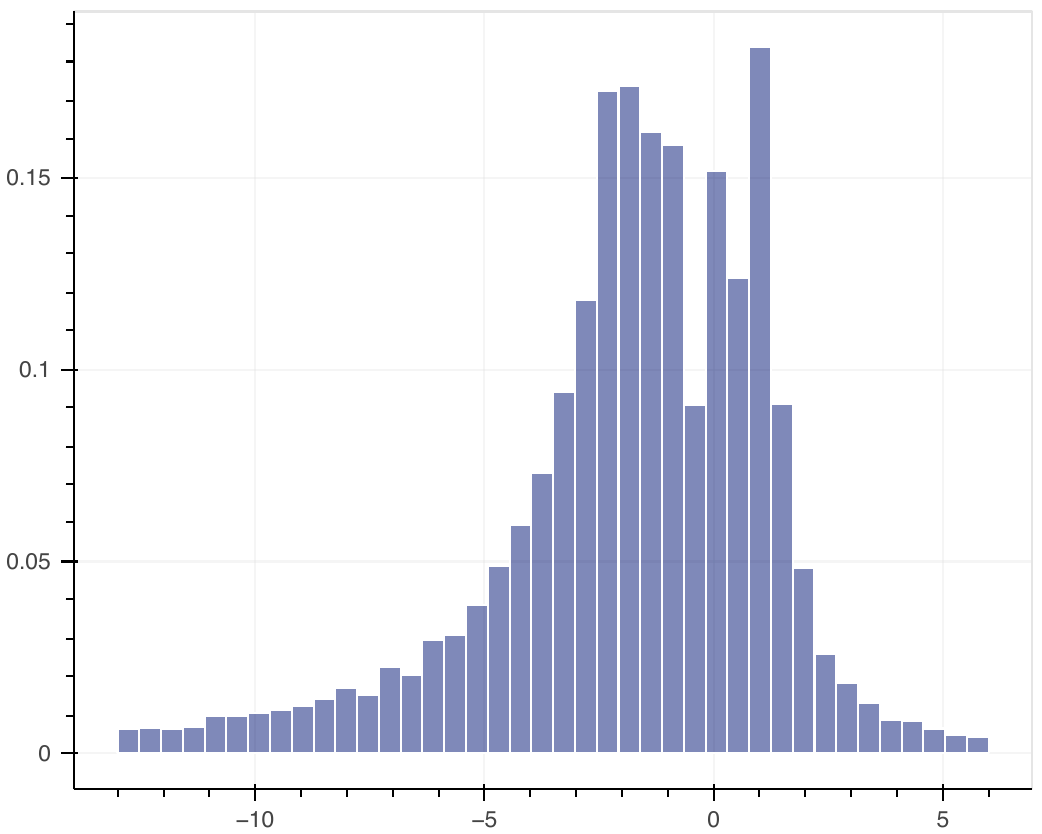}
        \captionsetup{justification=centering}
        \caption{$\varepsilon = 0.8, \lfloor h \left( \varepsilon \right) \rfloor = 2, \delta \left( \varepsilon \right) \approx 0.94$}
        \label{fig:hist_eps_0.8}
    \end{subfigure}
    \hfill
    \begin{subfigure}[b]{0.24\textwidth}
        \centering
        \includegraphics[width=\linewidth]{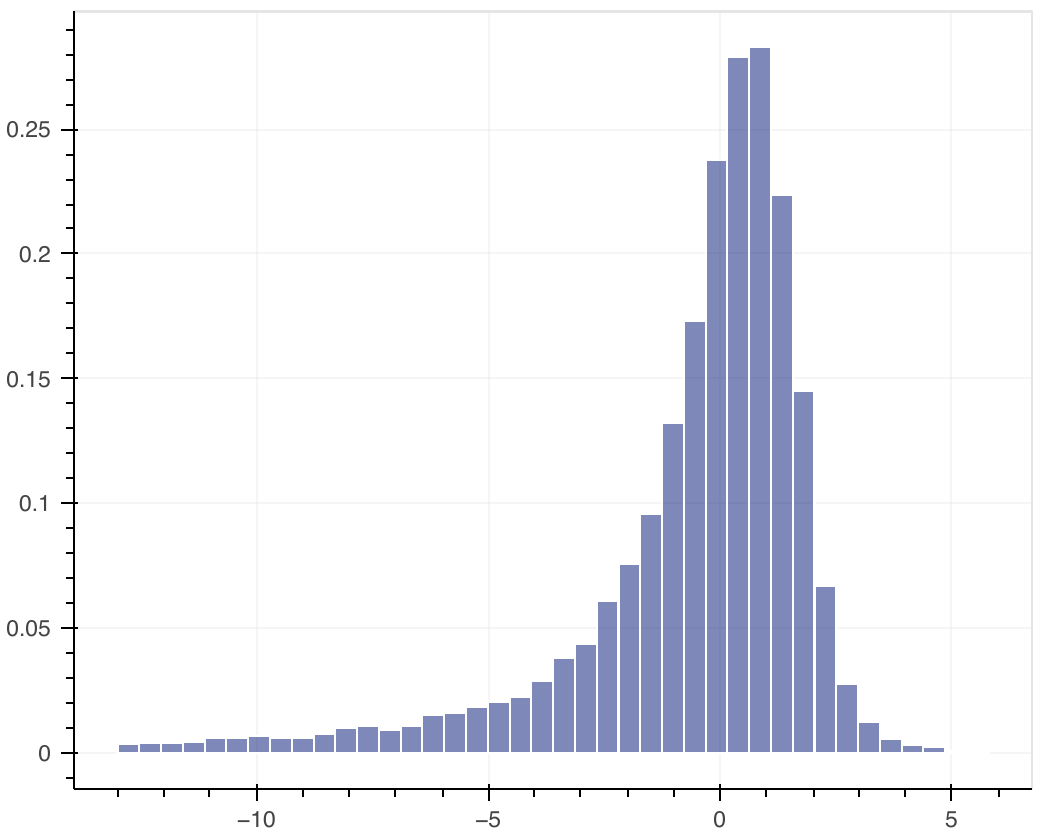}
        \captionsetup{justification=centering}
        \caption{$\varepsilon = 0.1, \lfloor h \left( \varepsilon \right) \rfloor = 19, \delta \left( \varepsilon \right) \approx 0.55$}
        \label{fig:hist_eps_0.1}
    \end{subfigure}
    \hfill
    \begin{subfigure}[b]{0.24\textwidth}
        \centering
        \includegraphics[width=\linewidth]{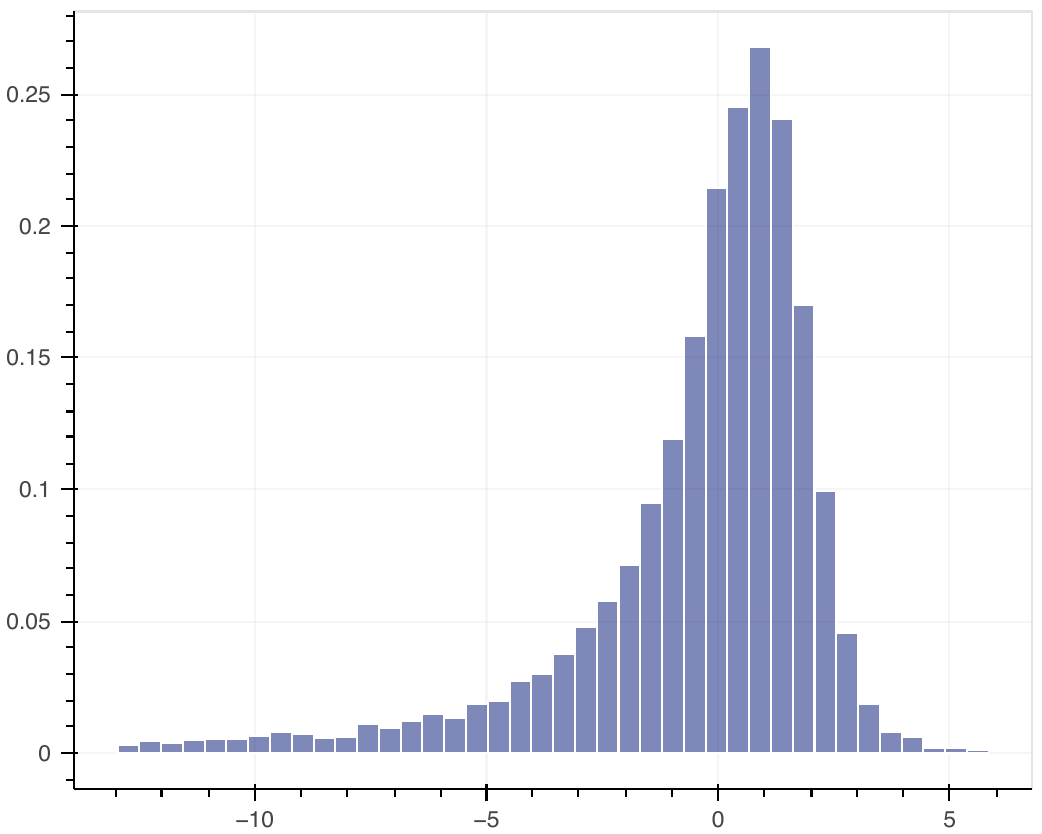}
        \captionsetup{justification=centering}
        \caption{$\varepsilon = 0.001, \lfloor h \left( \varepsilon \right) \rfloor = 5995, \delta \left( \varepsilon \right) \approx 0.17$}
        \label{fig:hist_eps_0.001}
    \end{subfigure}
    \caption{Histograms of the distribution of the normalized SGD sequence $x^{\varepsilon}_{\lfloor h \left( \varepsilon \right) \rfloor}/\delta(\varepsilon)$ for the objective function $f(x) = x^{3}/3$, initial value $x_{0} = \delta(\varepsilon)/2$, and noise $\xi$ following a Pareto distribution $\alpha = 1.7$, computed for various values of $\varepsilon$. As $\varepsilon$ decreases, the distribution stabilizes, which confirms the appropriateness of the normalization by the scaling functions $h(\varepsilon)$ and $\delta(\varepsilon)$ specified in equations~\eqref{eq_for_hdelta}~and~\eqref{def_normalized_seq}.}
    \label{fig_hist_one_step}
\end{figure}

\subsection{Leaving Neighborhood of a Sharp Maximum}\label{sec_sharp_max}
In this section, we consider the issue of escaping from the neighborhood of the sharp maximum. Being at the maximum point, the SGD sequence  can randomly go either to the right or to the left of this maximum. Thus, we will be interested in the probability of going beyond the right (or left) boundary of the neighborhood of the sharp maximum. 

Without loss of generality, we will assume that the maximum is located at zero. In this section, we assume that for some $\delta > 0$, the function $f(x)$ in the $\delta$-neighborhood of the maximum has a V-shape
\begin{equation}\label{eq_sharp_max_3}
f^{\prime}(x)=\left\{\begin{array}{l}
c_l, x \in (-\delta, 0), \\
-c_r, x \in[0, \delta),
\end{array}\right.
\end{equation}
where $c_l, c_r > 0$, $\delta>0$.

If the starting point is bounded away from zero—that is, there exists $\delta' > 0$ such that $x_0^{\varepsilon} \geq \delta'$ for all sufficiently small $\varepsilon > 0$—then the iterates immediately fall within the framework of Section~\ref{sec_suit_time_scal}, and the SGD sequence converges to the minimum of the basin of attraction containing $x_0^{\varepsilon}$. In this section, we consider the case where $x_0^{\varepsilon} \ra 0 $ when $\varepsilon \downarrow 0.$ For definiteness, we will assume that $x_0^{\varepsilon} \geqslant 0$ (in the case where $x_0^{\varepsilon} <0$ we need to redefine the first  jump in \eqref{eq_RRW} to $X_1 = \xi_1 - c_l$). 

We have
\begin{equation}\label{eq_sharp_max_2}
x_n^{\varepsilon}= x_0^{\varepsilon}  + \varepsilon \sum_{i=0}^{n-1}\left(\xi_{i+1}-f^{\prime}\left(x_i^{\varepsilon}\right)\right),
\end{equation}
where $\left\{\xi_i\right\}_{i \geqslant 1}$ are independent and identically distributed  random variables with $\mathbb{E} \xi_i=0$.
Let's define the first moment of leaving the $\delta$-neighborhood of zero
$$
N^{\varepsilon}:=\inf \left\{n \geqslant 1: x_n^{\varepsilon} \notin[-\delta, \delta]\right\}. 
$$
 We are interested in the limits of the probabilities that the SGD sequence will go to the left or right basin 
\begin{equation}\label{eq_sharp_max_9}
    \lim\limits_{\varepsilon\downarrow 0 } \p{x_{N^{\varepsilon}}^{\varepsilon} < -\delta},  \quad \lim\limits_{\varepsilon\downarrow 0 } \p{x_{N^{\varepsilon}}^{\varepsilon} > \delta}.
\end{equation}
For $\varepsilon >0, k \geqslant 1$ we define the moments of crossing the maximum
\begin{equation*}
    \begin{split}
        \tau_0^{\shortuparrow}(\varepsilon):=0, \ 
        \tau_{2 k-1}^{\shortdownarrow}(\varepsilon):=\inf \left\{j>\tau_{2k-2}^{\shortuparrow}(\varepsilon): x_j^{\varepsilon}<0\right\}, \\
        \tau_{2 k}^{\shortuparrow}(\varepsilon):=\inf \left\{j>\tau_{2 k-1}^{\shortdownarrow}(\varepsilon): x_j^{\varepsilon} \geqslant 0\right\}.
    \end{split}
\end{equation*}
From \eqref{eq_sharp_max_3} and \eqref{eq_sharp_max_2} we can see that between the (possibly infinite) moments of $\tau_{2 k-1}^{\shortdownarrow}(\varepsilon) \leqslant j < \tau_{2 k}^{\shortuparrow}(\varepsilon)$ $\left( \text{or } \tau_{2 k}^{\shortuparrow}(\varepsilon) \leqslant j < \tau_{2 k+1}^{\shortdownarrow}(\varepsilon)\right)$ sequence (by $j$) $\varepsilon^{-1} x_j^{\varepsilon}$ behaves like a random walk $
\sum_{i} (\xi_i - c_l)$  $\left(\text{or } \sum_{i} (\xi_i +c_r)\right),$ that goes to $-\infty$ $(\text{or } +\infty)$  by the law of large numbers. From here it is easy to show (see Lemma~\ref{bound_prob_lemma}) that 
\begin{equation*}
\text{ for all } \varepsilon>0 \ \ \p{N^{\varepsilon}<\infty} = 1   \text{ and } N^{\varepsilon} \pra \infty \text{ as }  \varepsilon\downarrow 0.
\end{equation*}

The key idea for evaluating the probabilities of \eqref{eq_sharp_max_9} is the definition of a runaway random walk (RRW) $X_0 := 0$ and 
\begin{equation}\label{eq_RRW}
X_n:=\left\{\begin{array}{l}
X_{n-1} + (\xi_n - c_l), \text{ if }  X_{n-1} < 0, \\
X_{n-1} + (\xi_n + c_r), \text{ if }  X_{n-1} \geqslant 0.
\end{array}\right.
\end{equation}
It is easy to see that if $X_{n-1} \geqslant 0,$ then the next jump has positive drift $\e{\xi_n + c_r} > 0,$ and hence by this jump RRW will  further move from zero to $+\infty$. Similarly, when $X_{n-1} <0,$ RRW will move away from zero to $-\infty.$
If $x_0^{\varepsilon} = 0$, then for all $\varepsilon>0$ up to moment $N_\varepsilon$, the sequences of stopping moments for SGD and RRW are the same 
\begin{equation*}
\skk{\tau_{2 k}^{\shortuparrow}(\varepsilon), \tau_{2 k+1}^{\shortdownarrow}(\varepsilon)}_{k \geqslant 0} = \skk{\tau_{2 k}^{\shortuparrow}, \tau_{2 k+1}^{\shortdownarrow}}_{k \geqslant 0},
\end{equation*}
where
\[
    \tau_0^{\shortuparrow}=0, \ 
    \tau_{2 k-1}^{\shortdownarrow}=\inf \left\{j>\tau_{2k-2}^{\shortuparrow}: X_j<0\right\}, \ 
\]
\[
    \tau_{2 k}^{\shortuparrow}=\inf \left\{j>\tau_{2 k-1}^{\shortdownarrow}: X_j \geqslant 0\right\}.
\]
If $x_0^{\varepsilon} \neq 0$, then the SGD and RRW sequences may cross zero at different time points, which can lead to their convergence in opposite directions. Nevertheless, the following theorem holds.
\begin{theorem}\label{th_sharp_max} For all $\delta >0 ,$  $x_{0}^{\varepsilon}\geqslant 0,$ $x_{0}^{\varepsilon} = o(\varepsilon)$ as $\varepsilon\downarrow 0$ we have
\begin{equation}\label{eq_sharp_max_4}
\begin{split}
\lim\limits_{\varepsilon\downarrow 0 } &\ \p{x_{N^{\varepsilon}}^{\varepsilon} > \delta}   = \p{ \lim\limits_{n \ra \infty} X_n  = + \infty} = \sum_{k \geqslant 0} \left( \p{  \tau_{2 k}^{\shortuparrow} < \infty} -\p{\tau_{2 k +1}^{\shortdownarrow} <\infty}\right), \\ 
% & \leqslant 1 - p_{\shortdownarrow} + \frac{p_{\shortdownarrow} p_{\shortuparrow}}{1- p_{\shortdownarrow} p_{\shortuparrow}} , \\
\lim\limits_{\varepsilon\downarrow 0 } &\ \p{x_{N^{\varepsilon}}^{\varepsilon} < -\delta}  = \p{ \lim\limits_{n \ra \infty} X_n  = - \infty}  = \sum_{k \geqslant 0} \left( \p{  \tau_{2 k+1}^{\shortdownarrow} < \infty} -\p{\tau_{2 k +2}^{\shortuparrow} <\infty}\right). \\ 
% & \leqslant  \frac{p_{\shortdownarrow} }{1- p_{\shortdownarrow} p_{\shortuparrow}}, \\ 
\end{split}
\end{equation}
\end{theorem}

As shown in Theorem~\ref{th_sharp_max}, the exit probabilities of SGD from a neighborhood of a maximum can be expressed in terms of exit times for RRW. We further reduce the analysis of these exit times to two  one-dimensional random walks with positive and negative drift. Since such processes are standard and their exit time asymptotics have been extensively studied \citep{Borovkov:20013, gut2008, siegmund1985}, we can leverage existing results to derive the required estimates for our setting. At the same time, this highlights the substantial increase in complexity that arises in the multidimensional setting: boundary crossing problems for random walks become significantly more challenging in higher dimensions, and the corresponding methodology is far less developed than in the one-dimensional case.

\begin{remark}\label{rk_sharp_max}
\begin{itemize}
\item[\namedlabel{rk_sm_i}{$(i)$}] For all $k\geqslant 1$ we have
\begin{equation}\label{eq_sharp_max_1}
\begin{split}
   \p{ \tau_{2 k-1}^{\shortdownarrow} < \infty} \leqslant p_{\shortdownarrow}^k p_{\shortuparrow}^{k-1}, \quad \p{ \tau_{2 k}^{\shortuparrow} < \infty} \leqslant (p_{\shortdownarrow} p_{\shortuparrow})^k,  
\end{split}
\end{equation}
where
\[
    \begin{split}
    p_{\shortdownarrow}:= \p{\inf\limits_{n \geqslant 1} \skk{\sum\limits_{k=1}^n \left(\xi_k +c_r\right)} < 0}, \quad p_{\shortuparrow}:=\p{\sup\limits_{n\geqslant 1} \skk{\sum\limits_{k=1}^n \left(\xi_k -c_l\right)} \geqslant 0}
    \end{split}
\]
--  probabilities that a random walk with positive $(c_r >0)$ or negative $(-c_l < 0)$ drift will ever cross zero. 

\item[\namedlabel{rk_sm_iii}{$(ii)$}] If the distribution of $\xi_1$ is exponential on the left and right semi-axes 
    \begin{equation*}
\begin{split}
& \p{\xi_1 \geqslant t}=q\exp\{-\alpha t\}, \ \alpha>0, \ t\geqslant 0, \\ 
& \p{\xi_1<t}=r\exp\{\beta t\},\  \beta>0,\  t<0,
 \end{split}
\end{equation*}
that is $r+q=1$ and $ \e{\xi_1}=q/\alpha-r/\beta=0$, then equalities hold in \eqref{eq_sharp_max_1} and
\begin{equation}\label{eq_sharp_max_5}
p_{\shortdownarrow}= 1-\mu_{\shortdownarrow} / \beta, \quad p_{\shortuparrow}=1-\mu_{\shortuparrow} / \alpha,
\end{equation}
where $-\mu_{\shortdownarrow} $ is the only negative root (in terms of $\lambda$) of equation $$\e{\exp \left\{\lambda (\xi_1 + c_r)\right\}}=e^{\lambda c_r}\left(\frac{r \beta}{\lambda+\beta}+\frac{q \alpha}{\alpha-\lambda}\right)=1,$$  $\mu_{\shortuparrow} $ is the only positive root of equation  $$\e{\exp \left\{\lambda (\xi_1 - c_l)\right\}}=e^{-\lambda c_l}\left(\frac{r \beta}{\lambda+\beta}+\frac{q \alpha}{\alpha-\lambda}\right)=1.$$
\end{itemize}
\end{remark}

\begin{corollary}\label{cl_sharp_max}
Let the conditions of Theorem~\ref{th_sharp_max} be satisfied.
\begin{itemize}
    \item[\namedlabel{cl_sm_i}{$(i)$}] 
From Remark~\ref{rk_sharp_max}~\ref{rk_sm_i} we get the following inequalities: 
\begin{equation*}
\begin{split}
& \lim\limits_{\varepsilon\downarrow 0 } \p{x_{N^{\varepsilon}}^{\varepsilon} > \delta}   \leqslant 1 - p_{\shortdownarrow} +  \frac{p_{\shortdownarrow} p_{\shortuparrow} }{1- p_{\shortdownarrow} p_{\shortuparrow}}, \\ & \lim\limits_{\varepsilon \downarrow 0 }   \p{x_{N^{\varepsilon}}^{\varepsilon} < -\delta}    
 \leqslant \frac{p_{\shortdownarrow} }{1- p_{\shortdownarrow} p_{\shortuparrow}}.
\end{split}
\end{equation*}
From this, it is clear that both limiting probabilities can be strictly less than $1$, allowing SGD to enter each of the neighboring basins with positive probability.

 \item[\namedlabel{cl_sm_i_exp}{$(ii)$}] 
From Remark~\ref{rk_sharp_max}~\ref{rk_sm_i}  and  \ref{rk_sm_iii} in the double-exponential case  we get the following  equalities in the double-exponential case: 
\begin{equation*}
\begin{split}
& \lim\limits_{\varepsilon\downarrow 0 } \p{x_{N^{\varepsilon}}^{\varepsilon} > \delta}   =    \frac{1 - p_{\shortdownarrow} }{1- p_{\shortdownarrow} p_{\shortuparrow}}, \\ 
& \lim\limits_{\varepsilon \downarrow 0 }   \p{x_{N^{\varepsilon}}^{\varepsilon} < -\delta}    
= \frac{p_{\shortdownarrow} (1  -  p_{\shortuparrow} )}{1- p_{\shortdownarrow} p_{\shortuparrow}}.
\end{split}
\end{equation*}

\item[\namedlabel{cl_sm_ii}{$(iii)$}] Let also the conditions of Theorem~\ref{t.1} or Theorem~\ref{t.2} be satisfied, where the condition on the initial point is replaced by $x_0^{\varepsilon} \geq M_r, |x_0^{\varepsilon} - M_r| = o(\varepsilon).$ Then for all $\delta'>0$
\begin{equation}\label{eq_sharp_max_13}
   \lim_{\varepsilon \downarrow 0} \p{ |x_{\lfloor n_\varepsilon \rfloor}^{\varepsilon} - m | < \delta' } =   \p{ \lim\limits_{n \ra \infty} X_n  = - \infty}. 
\end{equation}
Thus, with positive probability the sequence of SGD starting from one basin but close to the maximum point $M_r$ can jump over the maximum and converge to the minimum $m$ from another basin in time $n_{\varepsilon} \in \mathcal{N}_{\mathbf{H_1}} $ or $ \mathcal{N}_{\mathbf{H_2}} $. 
\end{itemize}
\end{corollary}

Following Corollary~\ref{cl_sharp_max}, we illustrate the exit probabilities of SGD from a neighborhood of the sharp maximum in the double-exponential noise case. In this setting, the limiting exit probabilities as $\varepsilon\downarrow 0$ can be computed exactly using the formulas in Corollary~\ref{cl_sharp_max}~\ref{cl_sm_i_exp}. The only additional numerical step is computing the roots $\mu_{\shortuparrow}$ and $\mu_{\shortdownarrow}$ from Remark~\ref{rk_sharp_max}~\ref{rk_sm_iii}, which were obtained using the standard bisection method. We also approximate the finite-$\varepsilon$ exit probabilities by Monte Carlo simulation, running SGD near the sharp maximum $10^5$ times. Table~\ref{tab:exit-probs-merged} reports the simulated exit frequencies (Simulated), the exact limiting values (Exact limit), and the corresponding theoretical upper bounds (Upper bound) from Corollary~\ref{cl_sharp_max}~\ref{cl_sm_i}. The bounds are tight for the parameters shown: the gap between the exact value and the corresponding bound is at most $3\times 10^{-3}$, and in several rows the values coincide at the reported precision. The simulated values are equally close to the exact values.

% Theoretical estimates are numerically tight: for each row, the two upper bounds nearly sum to $1$, indicating that they closely capture the total exit probability. In several rows, the simulated probabilities slightly exceed the corresponding theoretical bounds; this is not a contradiction but a consequence of the limited number of Monte Carlo trials. Since theoretical values are upper bounds, such small violations are expected due to statistical error when the number of samples is not large enough.

\begin{table}[ht]
\centering
\small
\renewcommand{\arraystretch}{1.1}

\begin{tabular*}{\linewidth}{@{\extracolsep{\fill}} c c c c c c c @{}}
\toprule
& \multicolumn{3}{c}{Left exit: $x_{N^{\varepsilon}}^{\varepsilon} < -\delta$}
& \multicolumn{3}{c}{Right exit: $x_{N^{\varepsilon}}^{\varepsilon} > \delta$} \\
\cmidrule(lr){2-4}\cmidrule(lr){5-7}
$\beta$ & Simulated & Exact limit & Upper bound & Simulated & Exact limit & Upper bound \\
\midrule
0.1 & 0.4798 & 0.4793 & 0.4823 & 0.5202 & 0.5207 & 0.5222 \ \\
0.25 & 0.4506 & 0.4492 & 0.4518 & 0.5494 & 0.5508 & 0.5519 \ \\
0.5 & 0.3957 & 0.3959 & 0.3977 & 0.6043 & 0.6041 & 0.6048 \ \\
0.75 & 0.3427 & 0.3401 & 0.3414 & 0.6573 & 0.6599 & 0.6604 \ \\
1.0 & 0.2823 & 0.2847 & 0.2857 & 0.7177 & 0.7153 & 0.7155 \ \\
1.5 & 0.1851 & 0.1869 & 0.1874 & 0.8149 & 0.8131 & 0.8132 \ \\
2.0 & 0.1153 & 0.1160 & 0.1163 & 0.8847 & 0.8840 & 0.8840 \ \\
3.0 & 0.0428 & 0.0425 & 0.0426 & 0.9572 & 0.9575 & 0.9575 \ \\
5.0 & 0.0060 & 0.0058 & 0.0058 & 0.9940 & 0.9942 & 0.9942 \ \\
\bottomrule
\end{tabular*}
\caption{
Simulated exit frequencies (Simulated), exact limiting exit probabilities as $\varepsilon\downarrow 0$ in the double-exponential case (Exact limit; Corollary~\ref{cl_sharp_max}~\ref{cl_sm_i_exp}), and the corresponding theoretical upper bounds for the limiting probabilities (Upper bound; Corollary~\ref{cl_sharp_max}~\ref{cl_sm_i}), shown for varying values of $\beta$. Parameters: $\alpha = 1$, $\varepsilon = 0.01$, $\delta = 1.0$, $c_r = 1.0$, $c_l = 5.0$.
}
\label{tab:exit-probs-merged}
\end{table}

% \subsection{Conclusion and discussions}\label{sec_conclusion}

% \textcolor{blue}{We did such-and-such.}

% \textcolor{red}{
% The results of Sections~\ref{sec_suit_time_scal}, \ref{sec_stic_crit_point} are generalized to the multidimensional case, but the derivations in this scenario become noticeably more complicated.
% %
% The conditions \ref{H1} and \ref{H2}
% %
% The proof of the results in Section~\ref{sec_stic_crit_point} can be reproduced almost verbatim, except that we use Taylor's formula and replace the absolute value with the norm in \eqref{eq_crit_esc_1}. }

% Note that the results of Section \ref{sec_suit_time_scal}
% can be extended to the multidimensional
% case, but this requires stronger smoothness assumptions on the function $f$.

% Let $m_1\in \mathbb{R}^d$ be the unique minimum of the function $f:\mathbb{R}^d \ra \mathbb{R}$ in some open ball $\textbf{B}\left(m_1, R \right) = \{x \in \mathbb{R}: \|x - m_1\|_2 < R \}$. Also assume that in the ball $\textbf{B}\left(m_1, R \right)$ the function $f$ is
% \begin{enumerate}
%     \item[\namedlabel{L}{$(1)$}] $L$-smooth, i.e., it is differentiable in $\textbf{B}\left(m_1, R \right)$ and for all $x,y \in \textbf{B}\left(m_1, R \right)$ 
%     $$\|\nabla f\left(x\right) - \nabla f\left(y\right) \|_2 \leqslant L \|x - y\|_2;$$
%     \item[\namedlabel{mu}{$(2)$}] $\mu$-strongly convex, i.e., for all $x,y \in \textbf{B}\left(m_1, R \right)$ 
%     $$
%     f\left(y\right) \geq f\left(x\right) + \langle \nabla f\left(x\right), y - x \rangle + \frac{\mu}{2} \| y - x \|_2^2.
%     $$
% \end{enumerate}

\section{Related work and discussion}\label{sec_lit_review}

\textbf{Diffusion approximation.} In works \cite{li2017, li2018}, the diffusion approximation of SGD is studied on a finite time interval. These results are suitable for analyzing the asymptotic behavior of the process over $O(1/\varepsilon)$ iterations, which is insufficient for the purposes of the present study. In \cite{li2022}, the diffusion approximation is considered on an infinite time interval, but only for the case of a strictly convex loss function. In contrast, we demonstrate that the results depend on the local properties of the loss function; therefore, it is not necessary to assume global strict convexity. 
\citet{hu2019quasi} use Large Deviations Theory in the anaylsis of SGD and \citet{baudel2023hill} consider metastable dynamics. However, these papers consider white noise which is very narrow example of condition \ref{H2} which leads to a larger times to escape a basin of local mimima.  As we mentioned in the introduction \citet{simsekli2019tail} argue that SGD can be viewed as a discrete version of L\'{e}vy driven Langevin equations. Thus, there is motivation to consider heavy-tailed noise. \citet{nguyen2019first} consider L\'{e}vy-driven SDE and show that under certain conditions an Euler discretization and continuous process have close tail-probabilities for the first time to leave a basin for a minimum. In our paper we firstly take different model with potentially discontinuous loss function and less restrictive conditions on the distribution of errors. Secondly, we take a closer look at time regions when SGD either converging to a minimum or is still in a neighborhood of other critical points.

\textbf{Statistical physics methods.} The works by \citet{mignacco2020dynamical,veiga2024stochastic} employ a continuous approximation of stochastic gradient dynamics and adopt settings that allow deriving analytical expressions for (among other things) generalization error. They further empirically demonstrate the consistency of their theoretical framework with practical observations. In contrast to these works, our study focuses on the dynamics of idealized stochastic gradient descent, deriving limit theorems supported by rigorous probabilistic analysis.

In another study, \cite{mignacco2022effective}, the authors provide a phenomenological description of the noise arising in SGD algorithms applied to binary classification tasks for Gaussian mixtures. Their analysis reveals that the noise admits a theoretical characterization via an effective temperature, derived through the fluctuation-dissipation theorem from statistical physics. In contrast to this phenomenological approach, our framework abstracts away from complex noise structures by assuming additive perturbations, enabling theoretical analysis under minimal distributional assumptions.

\textbf{Diagonal Linear Networks.} As investigated by \citet{berthier2023incremental}, these models exhibit a progressive feature activation mechanism during training. Initialized with near-zero weights, these models tend to remain near partially activated, saddle-like configurations for extended durations before transitioning toward states with more complete feature engagement. In the initial stages of training, the resulting representations are highly sparse; however, over time, an increasing number of coordinates become active, leading to denser final solutions. The number of effectively active features depends on the training duration: longer training results in reduced sparsity. These dynamics of gradual transitions between metastable states bear some limited resemblance to the behavior we analyze in the context of stochastic gradient descent, but our system differs notably in that it may revisit previously attained equilibria.

\textbf{Dynamics near critical points.}
In the studies by 
\citet{ziyin2023} and \citet{ziyin2024}, the  authors investigate conditions under which SGD converges to a maximum, contradicting intuition derived from GD. However, it is easy to observe that all claims are based on the toy example $L(x,\xi) = \xi x^2/2$, where SGD reduces to
    \[
        x_t = \prod\limits_{k=1}^{t} (1 - \lambda \xi_k)x_0.
    \]
    The analysis then relies on the simple fact that products of random variables with an expectation greater than one ($\e{1 - \lambda \xi_k} > 1$) can still converge to zero if the expectation of their logarithm is negative ($\e{ \ln (1 - \lambda \xi_k} < 0$). In the context of SGD, this implies convergence to a maximum.
In the current paper, we do not assume toy functions; instead, we present results on lingering at saddle points depending on the number of zero derivatives and moment conditions of the stochastic noise.

\textbf{Large deviation theory.}
Let us consider papers in which asymptotics in SGD is studied by methods of the theory of large deviations. In the works \citet{Azizian2024LDP}, \citet{Bajovic2023LDP}   the random variables \(\xi_k\) are sub-Gaussian; in \citet{Hult2025LDP}
the random variables $\xi_k$  have a finite exponential moment.  In all the above mentioned  works, $\xi_k$ can depend  on the position of SGD.
Obviously, such random variables represent a special case when the condition \ref{H2} holds.   
It is worth noting that, in principle, their large deviation methods allow studying the asymptotic behavior in SGD only when 
\(
\e{e^{c|\xi_k|}} < \infty
\)
for all \(c > 0\).  So called light tail case. Although the authors of \citet{Azizian2024LDP} and \citet{Bajovic2023LDP}  impose an even stronger condition.  
It is easy to see that this condition is not satisfied by any random variables from the class \ref{H1} and by "most" random variables from the class \ref{H2}.  
As can be observed from the works \citet{wang2021eliminating}, \citet{imkeller2008metastable}, \citet{simsekli2019tail, csimcsekli2019heavy}, the overall dynamics of SGD will be significantly different in the case of light tails. 

%\section{Conclusion} In this work, we have shown that on certain time scales the stochastic dynamics of SGD strongly depends on the algorithm's initialization. Specifically, if the initial point is in the basin between two local maximum $\left(M_{r-1}, M_r\right)$ and is separated from the maxima, then in time $\mathcal{O}\left(n_{\varepsilon}\right)$ we descend into a small neighborhood of the minimum $m_r \in \left(M_{r-1}, M_r\right)$ (Theorems \ref{t.1}, \ref{t.2}). However, if the initial point is sufficiently close, say, to $M_r$, then in the same amount of time we may not leave the small neighborhood of $M_r$ at all (Theorems \ref{th1_crit_point}, \ref{th2_crit_point}, Remark \ref{rem_crit_point}), or with positive probability we may descend into the neighborhood of the minimum $m_{r+1} > M_r$ (Theorem \ref{th_sharp_max}, Corollary \ref{cl_sharp_max}).

\section{Conclusion} 
In this work, we analyze trajectories of Stochastic Gradient Descent (SGD) with a constant step size within the framework of limit theorems in probability theory, specifically in the limit as the step size tends to zero.  All presented results are rigorously proven under two distinct scenarios: the absence of a second moment \eqref{H1} and the presence of a second moment \eqref{H2} for the noise distribution. We demonstrate that on certain time scales, the stochastic dynamics of SGD depend critically on the algorithm's initialization and the noise distribution properties. Specifically, if the initial point lies within the basin of a single local minimum $m$ located between two local maxima, $M_{l}$ and $M_r$, and is separated from these maxima, then within $n_{\varepsilon}$ iterations, where $n_{\varepsilon} \in \mathcal{N}_{H}$, the SGD trajectory predictably descends into a small neighborhood of the minimum $m \in \left(M_{l}, M_r\right)$ (Theorems~\ref{t.1}, \ref{t.2}).
However, to ensure almost sure convergence, the number of iterations must satisfy specific constraints. When the noise distribution possesses a second moment, we beleive that these constraints take the form $1/\varepsilon \ll n_{\varepsilon} \ll 1/\varepsilon^2$, independently of the tail of the noise distribution. It is interesting to compare this observation with the classical Robbins--Monro procedure with decaying step sizes.

Furthermore, if the initial point is sufficiently close, for instance, to $M_r$, then within the same time frame, the trajectory may not leave the small neighborhood of $M_r$ at all (Theorems \ref{th1_crit_point_simple}, \ref{th2_crit_point_simple}, Remark \ref{rem_crit_point}). 
Our investigation has led to a hypothesis regarding the asymptotics of the passage time through a neighborhood of an inflection point. In the presence of a second moment, this time is of the order $\varepsilon^{\frac{-2K}{K+1}}$, where $K+1$ denotes the order of the first non-zero derivative at the inflection point. Additionally, in a rather restrictive scenario involving a sharp maximum, we derived the probability of jumping over the maximum. Consequently, we show that if starting sufficiently close to the maximum, then with positive probability, the trajectory may descend into the basin of another minimum $m' > M_r$ (Theorem \ref{th_sharp_max}, Corollary \ref{cl_sharp_max}).  Therefore, our work provides insights into the appropriate number of iterations for  SGD.

\bibliography{articles}
\bibliographystyle{icml2025}

\newpage
\appendix
\onecolumn

\section{Technical Results and Proofs}\label{sec_appendix}

The appendix is organized as follows. In Sections \ref{sec_ap_1} and \ref{sec_ap_2} (for the cases with infinite and finite second moments, respectively), we present results that are, on the one hand, of a technical nature, but on the other hand, these results determine the main conditions imposed on the random variables $\xi_i$. The remaining part of the appendix follows the structure of the main part of the paper: in Section \ref{sec_ap_3}, we provide the proofs of the main statements regarding Suitable Time Scaling; in Section \ref{sec_ap_4}, we present the precise formulation and proofs of the main statements concerning Sticking to a Critical Point; and in Section \ref{sec_ap_5}, we provide the proofs of the main statements related to Leaving Neighborhood of a Sharp Maximum.

Here and after we use $L(u)$ to denote slowly varying at infinity function \cite{bingham1989regular}. Based on it, we also define slowly varying at infinity function $\widetilde{L}(u)$ such that
\beq\label{22.10.2}
\lim\limits_{u\rightarrow\infty}L(u)\widetilde{L}(uL(u))=1.
\eeq
Let us point out that $\widetilde{L}(u)$ is uniquely determined up to asymptotic equivalence
(see, for example, \citet[Theorem 1.5.13]{bingham1989regular}).

\subsection{Infinite Second Moment}\label{sec_ap_1}
For citation convenience let us formulate the result that belongs to \citet[theorem 3.1]{anh2021marcinkiewicz}.
\begin{lemma}[\citet{anh2021marcinkiewicz}] \label{anh2021_t31}
Let condition \ref{H1} hold and $L(u) \geqslant 0$ is slowly varying at infinity function such that $\e{|\xi_1|^\alpha L^\alpha(|\xi_1|)}<\infty$. Then the following holds: 
$$
\lim\limits_{n \to \infty}
\dfrac{\max\limits_{1\leqslant k \leqslant n}\left|\sum\limits_{r=1}^k\xi_r\right|}{b_n}=0 \text{ a.s.,}
$$
for $b_n:=n^{1/\alpha}\widetilde{L}(n^{1/\alpha})$. 
\end{lemma}

\begin{lemma} \label{l.01} Let condition \ref{H1} hold and $L(u) \geqslant 0$ is slowly varying at infinity function such that
\beq\label{22.10.1}
u L(u) \text{ is increasing for } u\geqslant 0 \text{ and } \e{|\xi_1|^\alpha L^\alpha(|\xi_1|)}<\infty.
\eeq
Then for $\hat{n}_\varepsilon:=\left(\frac{1}{\varepsilon}\right)^\alpha L^\alpha\left(\frac{1}{\varepsilon}\right)$ as $\varepsilon\downarrow 0$ the following statements are true:
%  Let condition \ref{H1} hold and let
% $\hat{n}_\varepsilon:=\left(\frac{1}{\varepsilon}\right)^\alpha L^\alpha\left(\frac{1}{\varepsilon}\right)$.
% Then for $\varepsilon\downarrow 0$ the following  statements hold:
\begin{enumerate}
\item[\namedlabel{i_l.01}{$(i)$}] $\hat{n}_\varepsilon=o\left(\frac{1}{H(1/\varepsilon)}\right)$;
\item[\namedlabel{ii_l.01}{$(ii)$}] 
$\lim\limits_{\varepsilon\downarrow 0}
\dfrac{\max\limits_{1\leqslant k \leqslant \hat{n}_\varepsilon}\left|\sum\limits_{r=1}^k\xi_r\right|}{b_{\hat{n}_\varepsilon}}=0$ a.s.,
for $b_{\hat{n}_\varepsilon}:=\hat{n}_\varepsilon^{1/\alpha}\widetilde{L}(\hat{n}_\varepsilon^{1/\alpha})$;
\item[\namedlabel{iii_l.01}{$(iii)$}]  $b_{\hat{n}_\varepsilon}\sim \frac{1}{\varepsilon}$.
\end{enumerate}
\end{lemma}

\begin{proof}[Proof of Lemma~\ref{l.01}]
Let us prove \ref{i_l.01}.
It is easy to see that it is sufficient to prove equality
\beq\label{22.10.12}
\lim_{u\rightarrow\infty}L^\alpha(u)\widehat{L}(u)=0.
\eeq
We use inequality \eqref{22.10.1} and the fact that function $u^\alpha L^\alpha(u)$ is increasing for $u\geqslant 0$ to prove
\beq\label{22.10.11}
\infty>\e{|\xi_1|^\alpha L^\alpha(|\xi_1|)}\geq
\sum\limits_{u=1}^\infty|u|^\alpha L^\alpha(|u|)(H(u)-H(u+1)).
\eeq
It follows from Lemma \ref{l.a1} that, for $u\rightarrow\infty$, we have
$$
|u|^\alpha L^\alpha(|u|)(H(u)-H(u+1))\sim \frac{1}{|u|}L^\alpha(|u|)\widehat{L}(u).
$$
Therefore condition (\ref{22.10.12}) is a necessary condition for the convergence of the series on the right-hand side of inequality (\ref{22.10.11}).

Clearly \ref{ii_l.01} follows from Lemma \ref{anh2021_t31}. Let us prove \ref{iii_l.01}.
It follows from condition \eqref{22.10.2} that for $\varepsilon\downarrow0$
$$
b_{\hat{n}_\varepsilon}=\hat{n}^{1/\alpha}_\varepsilon \widetilde{L}(\hat{n}^{1/\alpha}_\varepsilon)
=\frac{1}{\varepsilon}L\left(\frac{1}{\varepsilon}\right)
\widetilde{L}\left(\frac{1}{\varepsilon}L\left(\frac{1}{\varepsilon}\right)\right)
\sim\frac{1}{\varepsilon}.
$$
\end{proof}

\begin{corollary} \label{c.1}Let condition \ref{H1} hold and let $n_\varepsilon \geqslant 0$ be such that $n_\varepsilon \leqslant \hat{n}_\varepsilon$, where $\hat{n}_\varepsilon$ is defined in Lemma \ref{l.01}. 
Then it follows from Lemma \ref{l.01} that, as $\varepsilon\downarrow 0$, the equality
\beq\label{24.10.2}
\max\limits_{1\leqslant k \leqslant n_\varepsilon}\left|\sum\limits_{r=1}^k\xi_r\right|
=o\left(\frac{1}{\varepsilon}\right)
\eeq
holds a.s..
\end{corollary}

\begin{lemma} \label{l.03} Let condition \ref{H1} hold and 
$n_\varepsilon \geqslant 0$ satisfies the following:
\begin{enumerate}
    \item[\namedlabel{i_l.03}{$(i)$}] there exists $\varepsilon_0>0$ such that function $n_\varepsilon$ is decreasing on the interval $(0,\varepsilon_0]$;
    \item[\namedlabel{ii_l.03}{$(ii)$}] $n_\varepsilon=o\left(\frac{1}{H(1/\varepsilon)}\right)$ for $\varepsilon\downarrow 0$;
    \item[\namedlabel{iii_l.03}{$(iii)$}] $\lim\limits_{\varepsilon\downarrow 0}\varepsilon n_\varepsilon=\infty$.
\end{enumerate}

Then for every $\delta>0$
$$
\lim\limits_{\varepsilon\downarrow 0}\p{\varepsilon\max\limits_{1\leqslant k \leqslant n_\varepsilon}\left|\sum\limits_{r=1}^k\xi_r\right|>\delta}=0.
$$
\end{lemma}

\begin{proof}[Proof of Lemma~\ref{l.03}]
From condition \ref{ii_l.03} we get that for $\varepsilon\downarrow 0$
\beq\label{23.10.2}
n_\varepsilon \p{|\xi_1|\geqslant \frac{1}{\varepsilon}}=n_\varepsilon H\left(\frac{1}{\varepsilon}\right)=o(1).
\eeq

Since $\e{\xi_1}=0$ we have
\beq\label{23.10.3}
\left|\varepsilon n_\varepsilon \int_{-1/\varepsilon}^{1/\varepsilon} u dF(u)\right|
=\left|\varepsilon n_\varepsilon \int_{|u|\geqslant 1/\varepsilon} u dF(u)\right|
\leq\varepsilon n_\varepsilon \int_{|u|\geqslant 1/\varepsilon} |u| dF(u).
\eeq
We apply Karamata's theorem (see, for example, \citet[theorem 6.1]{buldygin2018generalized}) and we get
for $\varepsilon\downarrow 0$
% $$
% \int_{1/\varepsilon}^{\infty} u dF(u)=
% \int_{1/\varepsilon}^{\infty}\int_{0}^{u} dv dF(u)
% =\int_{1/\varepsilon}^{\infty}\int_{v}^{\infty}dF(u)dv,
% $$
% \beq\label{23.10.5}
% =\int_{1/\varepsilon}^{\infty}\overline{F}(v)dv\sim \frac{1}{1-\alpha}\left(\frac{1}{\varepsilon}\right)^{1-\alpha} L^+\left(\frac{1}{\varepsilon}\right)
% \eeq
% \textcolor{red}{================================}

\begin{align}
& \int_{1/\varepsilon}^{\infty} u dF(u) = -\int_{1/\varepsilon}^{\infty} u d\overline{F}(u) = -u\overline{F}(u) \mid_{1/\varepsilon}^{\infty} + \int_{1/\varepsilon}^{\infty} \overline{F}(u) du = 1/\varepsilon \overline{F}(1/\varepsilon) + \int_{1/\varepsilon}^{\infty} \overline{F}(u) du = \nonumber \\
& \left(\frac{1}{\varepsilon}\right)^{1-\alpha} L^+\left(\frac{1}{\varepsilon}\right) + \int_{1/\varepsilon}^{\infty} \overline{F}(u) du \sim \left(\frac{1}{\varepsilon}\right)^{1-\alpha} L^+\left(\frac{1}{\varepsilon}\right) + \frac{1}{1-\alpha}\left(\frac{1}{\varepsilon}\right)^{1-\alpha} L^+\left(\frac{1}{\varepsilon}\right) = \nonumber \\
& \left( 1 + \frac{1}{1-\alpha} \right) \left(\frac{1}{\varepsilon}\right)^{1-\alpha} L^+\left(\frac{1}{\varepsilon}\right) 
\end{align}
and similarly
\beq\label{23.10.7}
\int_{-\infty}^{ -1/\varepsilon} |u| dF(u)\sim \left( 1 + \frac{1}{1-\alpha} \right)\left(\frac{1}{\varepsilon}\right)^{1-\alpha}L^-\left(\frac{1}{\varepsilon}\right).
\eeq
From condition \ref{ii_l.03} and formulae (\ref{23.10.3})--(\ref{23.10.7}) we get for $\varepsilon\downarrow 0$ 
\beq\label{23.10.8}
\left|\varepsilon n_\varepsilon \int_{-1/\varepsilon}^{1/\varepsilon} u dF(u)\right|\leqslant \left( 1 + \frac{1}{1-\alpha} \right)n_\varepsilon H\left(\frac{1}{\varepsilon}\right)(1+o(1))=o(1).
\eeq
We apply Karamata's integral theorem and we get for $\varepsilon\downarrow 0$
$$
\int_{0}^{1/\varepsilon} u^2 dF(u)
=\int_{0}^{1/\varepsilon}\int_{0}^{u} 2v dv dF(u)
=\int_{0}^{1/\varepsilon}2v\int_{v}^{1/\varepsilon}dF(u) dv
$$
$$
=\int_{0}^{1/\varepsilon}2v
(F(1/\varepsilon)-F(v))dv\leq
\int_{0}^{1/\varepsilon}2v
\overline{F}(v)dv
$$
\beq\label{23.10.10}
=2\int_{0}^{1/\varepsilon}v^{1-\alpha}
L^+(v)dv
\sim\frac{2}{2-\alpha}\left(\frac{1}{\varepsilon}\right)^{2-\alpha} L^+\left(\frac{1}{\varepsilon}\right),
\eeq
and similarly
\beq\label{23.10.11}
\int_{-1/\varepsilon}^{0} u^2 dF(u)
\sim\frac{2}{2-\alpha}\left(\frac{1}{\varepsilon}\right)^{2-\alpha} L^-\left(\frac{1}{\varepsilon}\right)
\eeq
Now we apply condition \ref{ii_l.03} and formulae (\ref{23.10.10}) and (\ref{23.10.11}) and we get for $\varepsilon\downarrow 0$
\beq\label{23.10.12}
n_\varepsilon \varepsilon^2
 \int_{-1/\varepsilon}^{1/\varepsilon} u^2 dF(u)\leqslant 
 \frac{2}{2-\alpha}n_\varepsilon H\left(\frac{1}{\varepsilon}\right)(1+o(1))=o(1).
\eeq
From condition \ref{iii_l.03}, equality (\ref{23.10.2}), inequalities (\ref{23.10.8}), (\ref{23.10.12})
and \citet[theorem 5, page 211]{petrov1987limit} it follows that for $\delta>0$
\beq\label{23.10.14}
\lim\limits_{\varepsilon\downarrow 0}\p{\varepsilon\left|\sum\limits_{r=1}^{\lfloor n_\varepsilon \rfloor}\xi_r\right|>\frac{\delta}{3}}=0.
\eeq

We use \citet[theorem 22.5]{billingsley2017probability} and get
\beq\label{23.10.19}
\p{\max\limits_{1\leqslant k \leqslant n_\varepsilon}\varepsilon\left|\sum\limits_{r=1}^{k}\xi_r\right|>\delta}
\leqslant 3\max\limits_{1\leqslant k \leqslant n_\varepsilon}\p{\varepsilon\left|\sum\limits_{r=1}^{k}\xi_r\right|>\frac{\delta}{3}}.
\eeq

It follows from (\ref{23.10.14}) that for every $\epsilon>0$  there is such $\varepsilon^*\in(0,\varepsilon_0)$ that for every
$0<\varepsilon\leq\varepsilon^*$ the following inequality holds:
\beq\label{23.10.15}
\p{\varepsilon\left|\sum\limits_{r=1}^{\lfloor n_{\varepsilon} \rfloor}\xi_r\right|>\frac{\delta}{3}}<\frac{\epsilon}{3}.
\eeq
First, let us assume that function $n_\varepsilon$ is continuos on the interval $(0,\infty)$. Then for every $\lfloor n_{\varepsilon^*} \rfloor \leqslant k \leqslant \lfloor n_{\varepsilon} \rfloor$
there is such $\varepsilon\leq\varepsilon_k\leq\varepsilon^*$ that $n_{\varepsilon_k}=k$.
Therefore, we use (\ref{23.10.15}) to conclude that
\beq\label{23.10.17}
\frac{\epsilon}{3}>\max\limits_{\lfloor n_{\varepsilon^*} \rfloor \leqslant k\leqslant \lfloor n_{\varepsilon}\rfloor}
\p{\varepsilon_k\left|\sum\limits_{r=1}^{k}\xi_r\right|>\frac{\delta}{3}}>
\max\limits_{\lfloor n_{\varepsilon^*}\rfloor\leqslant k\leqslant \lfloor n_{\varepsilon}\rfloor}
\p{\varepsilon\left|\sum\limits_{r=1}^{k}\xi_r\right|>\frac{\delta}{3}}.
\eeq
Since $\varepsilon^*$ is fixed, for sufficiently small $\varepsilon$ 
the following inequality holds:
\beq\label{23.10.21}
\max\limits_{1\leqslant k< \lfloor n_{\varepsilon^*} \rfloor}\p{\varepsilon\left|\sum\limits_{r=1}^{k}\xi_r\right|>\frac{\delta}{3}}
<\frac{\epsilon}{3}.
\eeq
We use inequalities (\ref{23.10.19}), (\ref{23.10.17}) and (\ref{23.10.21}) to conclude that
for every $\delta>0$, $\epsilon>0$ and sufficiently small $\varepsilon$ the following inequality holds:
$$
\p{\max\limits_{1\leqslant k \leqslant n_\varepsilon}\varepsilon\left|\sum\limits_{r=1}^{k}\xi_r\right|>\delta}
\leqslant \epsilon.
$$
% \textcolor{red}{Now, let function $n_\varepsilon$ be discontinuous. It follows from monotonicity (condition~\ref{i_l.03}) that $n_\varepsilon$ has only countable number of the first order discontinuities. We can ``stitch'' discontinuities of $n_\varepsilon$ in a general continuous manner. We thus get a new function $n_\varepsilon^{c}$. We just proved for $n_\varepsilon^{c}$ that for every $\delta>0$, $\epsilon>0$ and sufficiently small $\varepsilon$ the following holds:
% $$
% \p{\max\limits_{1\leqslant k \leqslant n_\varepsilon^{c}}\varepsilon\left|\sum\limits_{r=1}^{k}\xi_r\right|>\delta}
% \leqslant \epsilon.
% $$
% We choose $\varepsilon$ small enough so that in addition $n_\varepsilon = n_\varepsilon^{c}$. Therefore, $\max\limits_{1\leqslant k \leqslant n_\varepsilon^{c}}\left|\sum\limits_{r=1}^{k}\xi_r\right| = \max\limits_{1\leqslant k \leqslant n_\varepsilon}\left|\sum\limits_{r=1}^{k}\xi_r\right|$ and the proof is complete.}

Let the function $n_{\varepsilon}$ be discontinuous. Then, using conditions \ref{i_l.03} -- \ref{iii_l.03}, it can be shown that there exists a continuous function $n_{\varepsilon}^{c}$ such that it satisfies conditions \ref{i_l.03} -- \ref{iii_l.03} and the inequality $n_{\varepsilon}^{c} > n_{\varepsilon}$ holds for all $\varepsilon > 0$. Therefore, according to the result just proven, for all $\delta > 0$, $\epsilon > 0$, and sufficiently small $\varepsilon$, we have
$$
\p{\max\limits_{1\leqslant k \leqslant n_\varepsilon}\varepsilon\left|\sum\limits_{r=1}^{k}\xi_r\right|>\delta} \leqslant \p{\max\limits_{1\leqslant k \leqslant n_\varepsilon^{c}}\varepsilon\left|\sum\limits_{r=1}^{k}\xi_r\right|>\delta}
\leqslant \epsilon.
$$
\end{proof} 

Define $\widehat{L}(u):=L^+(u)+L^-(u).$
\begin{lemma} \label{l.a1} For $u\rightarrow\infty$ equality
$$
H(u)-H(u+1)=\widehat{L}(u)\left(\frac{\alpha}{u^{\alpha+1}}+o\left(\frac{1}{u^{\alpha+1}}\right)\right)
$$
holds.
\end{lemma}

\begin{proof}
It is easy to see that
$$
H(u)-H(u+1)=\widehat{L}(u)\left(\frac{1}{u^\alpha}-\frac{1}{(u+1)^\alpha}\right)
$$
\beq\label{22.10.10}
+\frac{1}{(u+1)^\alpha}\widehat{L}(u)\left(1-\frac{\widehat{L}(u+1)}{\widehat{L}(u)}\right)
=:\widehat{L}(u)(f_1(u)+f_2(u)).
\eeq
For $u\rightarrow\infty$ we have
$$
f_1(u)=\frac{1}{u^\alpha}-\frac{1}{(u+1)^\alpha}
=\frac{1}{u^\alpha}\left(1-\frac{1}{\left(1+\frac{1}{u}\right)^\alpha}\right)
$$
\beq\label{22.10.7}
=\frac{1}{u^\alpha}\left(\frac{\alpha}{u}+o\left(\frac{1}{u}\right)\right)
=\frac{\alpha}{u^{\alpha+1}}+o\left(\frac{1}{u^{\alpha+1}}\right).
\eeq

Now we use the representation of slowly varying functions in exponent form (see, for example, \citet[theorem 1.2]{seneta1985regularly}) and, for $u\rightarrow\infty$, we get
\beq\label{22.10.3}
\frac{\widehat{L}(u+1)}{\widehat{L}(u)}=(1+o(1))\exp\left\{\int_u^{u+1}\frac{\epsilon(z)}{z}dz\right\}.
\eeq
Here $\epsilon(u)$ is a continuous function such that $\lim\limits_{u\rightarrow\infty}\epsilon(u)=0$.

Now we write the upper and lower bounds for the exponent on the right-hand side of (\ref{22.10.3}):
$$
\exp\left\{-\int_u^{u+1}\frac{|\epsilon(z)|}{z}dz\right\}\leq
\exp\left\{\int_u^{u+1}\frac{\epsilon(z)}{z}dz\right\}
\leq\exp\left\{\int_u^{u+1}\frac{|\epsilon(z)|}{z}dz\right\}
$$
$$
\exp\left\{-\max\limits_{z\in[u,u+1]}|\epsilon(z)|\ln\left(1+\frac{1}{u}\right)\right\}\leq
\exp\left\{\int_u^{u+1}\frac{\epsilon(z)}{z}dz\right\}
$$
\beq\label{22.10.5}
\leq\exp\left\{\max\limits_{z\in[u,u+1]}|\epsilon(z)|\ln\left(1+\frac{1}{u}\right)\right\}.
\eeq
Since $\max\limits_{z\in[u,u+1]}|\epsilon(z)|=o(1)$, $\ln\left(1+\frac{1}{u}\right)=\frac{1}{u}+o(\frac{1}{u})$
for $u\rightarrow\infty$, it follows from formulae (\ref{22.10.3}), (\ref{22.10.5}) that for $u\rightarrow\infty$
$$
\frac{\widehat{L}(u+1)}{\widehat{L}(u)}=\exp\left\{o\left(\frac{1}{u}\right)\right\}.
$$
Therefore, for $u\rightarrow\infty$
$$
f_2(u)=\frac{1}{(u+1)^\alpha}\left(1-\frac{\widehat{L}(u+1)}{\widehat{L}(u)}\right)
$$
\beq\label{22.10.8}
=\frac{1}{(u+1)^\alpha}\left(1-\exp\left\{o\left(\frac{1}{u}\right)\right\}\right)
=o\left(\frac{1}{u^{\alpha+1}}\right).
\eeq
Finally, the statement of the lemma follows directly from equalities (\ref{22.10.10}), (\ref{22.10.7}), (\ref{22.10.8}).
\end{proof}

\subsection{Finite Second Moment}\label{sec_ap_2}

% \textcolor{red}{Что означает условие \eqref{22.10.17} в терминах самой функции $L?$}

% Условие \eqref{22.10.17} эквивалентно $L(u) = o\left( (\ln \ln u)^{-\frac{1}{2}}\right).$
% \begin{proof} Let us prove that \eqref{22.10.17} follows from $L(u) = o\left( (\ln \ln u)^{-\frac{1}{2}}\right)$. Since $\widetilde{L}(u) L(u \widetilde{L}(u)) \ra 1$, we have
% $$
% \frac{\widetilde{L}(u)}{(\ln \ln u\widetilde{L}(u))^{\frac{1}{2}}} = \frac{\widetilde{L}(u) L(u \widetilde{L}(u))}{(\ln \ln u\widetilde{L}(u))^{\frac{1}{2}} L(u \widetilde{L}(u)) } \ra \infty.$$
% Next, we have
% \begin{equation*}
%     \ln \ln u\widetilde{L}(u) =  \ln \left(\ln u + \ln \widetilde{L}(u)\right) = \ln \ln u + \ln \left(1 + \frac{\ln \widetilde{L}(u)}{\ln u}\right)  \sim \ln\ln u,
% \end{equation*}
% where the last relation holds since for slowly varying function the following is true: 
% \begin{equation*}
%    \frac{\ln L(u)}{\ln u} \ra 0. 
% \end{equation*}
% Therefore,
% $$
% \lim\limits_{u\rightarrow\infty}\frac{\widetilde{L}(u)}{\sqrt{\ln\ln u}}=\infty
% $$
% and
% $$
% \lim\limits_{u\rightarrow\infty}\frac{\widetilde{L}(\sqrt{u})}{\sqrt{\ln\ln u}}= \lim\limits_{u\rightarrow\infty}\frac{\widetilde{L}(\sqrt{u})}{\sqrt{\ln \left(2\ln \sqrt{u}\right)}} =  \lim\limits_{u\rightarrow\infty}\frac{\widetilde{L}(\sqrt{u})}{\sqrt{\ln 2 + \ln\ln \sqrt{u}}} = \infty.
% $$
% \end{proof}

\begin{lemma} \label{l.02} Let condition \ref{H2} hold and $L(u) = o\left( (\ln \ln u)^{-\frac{1}{2}}\right)$. 
% Тогда для
% $\check{n}_\varepsilon=\left(\frac{1}{\varepsilon}\right)^2L^2\left(\frac{1}{\varepsilon}\right)$,
% где медленно меняющаяся на бесконечности функция $L(u)$ такая, что 
% \beq\label{22.10.17}
% \lim\limits_{u\rightarrow\infty}\frac{\widetilde{L}(\sqrt{u})}{\sqrt{\ln\ln u}}=\infty.
% \eeq
Then for
$\check{n}_\varepsilon=\left(\frac{1}{\varepsilon}\right)^2L^2\left(\frac{1}{\varepsilon}\right)$ the following holds:
$$
\lim\limits_{\varepsilon\downarrow 0}
\varepsilon\max\limits_{1\leqslant k \leqslant \check{n}_\varepsilon}\left|\sum\limits_{r=1}^k\xi_r\right|=0 \ \text{a.s.}
$$
\end{lemma}

\begin{proof}[Proof of Lemma~\ref{l.02}]
Let us prove that 
\beq\label{22.10.17}
\lim\limits_{u\rightarrow\infty}\frac{\widetilde{L}(\sqrt{u})}{\sqrt{\ln\ln u}}=\infty.
\eeq
follows from $L(u) = o\left( (\ln \ln u)^{-\frac{1}{2}}\right)$. Since $\widetilde{L}(u) L(u \widetilde{L}(u)) \ra 1$, we have
$$
\frac{\widetilde{L}(u)}{(\ln \ln u\widetilde{L}(u))^{\frac{1}{2}}} = \frac{\widetilde{L}(u) L(u \widetilde{L}(u))}{(\ln \ln u\widetilde{L}(u))^{\frac{1}{2}} L(u \widetilde{L}(u)) } \ra \infty.$$
Next, we have
\begin{equation*}
    \ln \ln u\widetilde{L}(u) =  \ln \left(\ln u + \ln \widetilde{L}(u)\right) = \ln \ln u + \ln \left(1 + \frac{\ln \widetilde{L}(u)}{\ln u}\right)  \sim \ln\ln u,
\end{equation*}
where the last relation holds since for slowly varying function the following is true: 
\begin{equation*}
   \frac{\ln L(u)}{\ln u} \ra 0. 
\end{equation*}
Therefore,
$$
\lim\limits_{u\rightarrow\infty}\frac{\widetilde{L}(u)}{\sqrt{\ln\ln u}}=\infty
$$
and
$$
\lim\limits_{u\rightarrow\infty}\frac{\widetilde{L}(\sqrt{u})}{\sqrt{\ln\ln u}}= \lim\limits_{u\rightarrow\infty}\frac{\widetilde{L}(\sqrt{u})}{\sqrt{\ln \left(2\ln \sqrt{u}\right)}} =  \lim\limits_{u\rightarrow\infty}\frac{\widetilde{L}(\sqrt{u})}{\sqrt{\ln 2 + \ln\ln \sqrt{u}}} = \infty.
$$
Next, it follows from law of iterated logarithm and (\ref{22.10.17}) that
$$
\lim\limits_{\varepsilon\downarrow 0}
\frac{\max\limits_{1\leqslant k \leqslant \check{n}_\varepsilon}\left|\sum\limits_{r=1}^k\xi_r\right|}{\sqrt{\check{n}_\varepsilon}\widetilde{L}(\sqrt{\check{n}_\varepsilon})}=0 \ \text{a.s.}
$$
Now, from condition (\ref{22.10.2}) for $\varepsilon\downarrow 0$ the following holds:
\begin{equation}
\sqrt{\check{n}_\varepsilon}\widetilde{L}(\sqrt{\check{n}_\varepsilon})
=\left(\frac{1}{\varepsilon}\right)L\left(\frac{1}{\varepsilon}\right)
\widetilde{L}\left(\frac{1}{\varepsilon}L\left(\frac{1}{\varepsilon}\right)\right) \sim \frac{1}{\varepsilon}. \label{L.A7_eq2}
\end{equation}
\end{proof}

\begin{remark} \label{r.1} let us point out that it follows from conditions (\ref{22.10.2}), (\ref{22.10.17}) that $\check{n}_\varepsilon=o\left(\frac{1}{\varepsilon}\right)^2$ for  $\varepsilon\downarrow 0$.
\end{remark}

\begin{lemma} \label{l.04} Let condition \ref{H2} hold and $n_\varepsilon$ satisfies conditions
\begin{enumerate}
    \item[\namedlabel{i_l.04}{$(i)$}] there exists $\varepsilon_0>0$ such that function $n_\varepsilon$ decreases on interval $(0,\varepsilon_0]$;
    \item[\namedlabel{ii_l.04}{$(ii)$}] $n_\varepsilon=o\left(\frac{1}{\varepsilon^2}\right)$ for $\varepsilon\downarrow 0$;
    \item[\namedlabel{iii_l.04}{$(iii)$}] $\lim\limits_{\varepsilon\downarrow 0}\varepsilon n_\varepsilon=\infty$.
\end{enumerate}
Then for every $\delta>0$
$$
\lim\limits_{\varepsilon\downarrow 0}\p{\varepsilon\max\limits_{1\leqslant k \leqslant n_\varepsilon}\left|\sum\limits_{r=1}^k\xi_r\right|>\delta}=0.
$$
\end{lemma}

\begin{proof}[Proof of Lemma~\ref{l.04}]
It follows from central limit theorem that for every $\delta>0$
$$
\lim\limits_{\varepsilon\downarrow 0}\p{\varepsilon\left|\sum\limits_{r=1}^{\lfloor n_\varepsilon \rfloor}\xi_r\right|>\delta}=0,
$$
The rest of the proof is analogous to the part of proof of Lemma \ref{l.03} 
which goes after formula (\ref{23.10.14}). Therefore, we leave them out.
\end{proof}

\subsection{Proofs of Main Results}

 Let us first state few technical lemmas.  For $h>0$ and $\varepsilon>0$  we define events
$$
A_h\left(\varepsilon\right):=\left\{\max\limits_{1\leqslant k \leqslant n_\varepsilon}\varepsilon\left|\sum\limits_{j=1}^k\xi_j\right|<h\right\}
$$
and
$$
B_{x,h}\left(\varepsilon\right):=\left\{\max\limits_{0\leqslant k \leqslant n_\varepsilon}x_k<x+h\right\}.
$$
\begin{lemma} \label{l.1}
For $0<\delta<M_r-m$, $x\in[m,M_r-\delta)$,
$0<\varepsilon<\frac{\delta}{10 C_{\max}}$ the following equality holds:
\beq\label{20.11.5}
A_{\frac{\delta}{5}}\left(\varepsilon\right)\cap B_{x,\delta}\left(\varepsilon\right)=A_{\frac{\delta}{5}}\left(\varepsilon\right).
\eeq
\end{lemma}

\begin{proof}
If $A_{\frac{\delta}{5}}\left(\varepsilon\right)=\varnothing$ then it is easy to see that (\ref{20.11.5})
holds. Assume that $A_{\frac{\delta}{5}}\left(\varepsilon\right)\neq\varnothing$.
Let us define a sequence of stopping times:
$$
\tau_0:=0, \ \ \
\tau_{2v-1}:=\inf\{\tau_{2v-2}< k \leqslant n_\varepsilon:x_k<m\}\wedge \lfloor n_\varepsilon \rfloor,
$$
$$
\tau_{2v}:=\inf\{\tau_{2v-1}< k \leqslant n_\varepsilon:x_k\geqslant m\}\wedge \lfloor n_\varepsilon \rfloor.
$$
Here $1\leqslant v \leqslant v_{\max}$, $v_{\max}:=\min\{v\in \mathbb{N}:\tau_{2v}= \lfloor n_\varepsilon \rfloor\}$, $\inf\{\varnothing\}=\infty$.

Let us point out that on event $A_{\frac{\delta}{5}}\left(\varepsilon\right)$ the following equality holds:
\beq \label{21.10.1}
x_k^{\varepsilon}<x+\delta 
\eeq
for $k=\tau_0$ and $\tau_{2v-1}\leqslant k < \tau_{2v}$, $1\leqslant v \leqslant v_{\max}$.

It is easy to that on event $A_{\frac{\delta}{5}}\left(\varepsilon\right)$ the following holds:
\beq \label{21.10.2}
\varepsilon\left|\sum\limits_{j=l_1}^{l_2}\xi_j\right|<\frac{2\delta}{5}
\eeq
for every $1\leqslant l_1\leqslant l_2 \leqslant n_\varepsilon$. Hence for $0<\varepsilon<\frac{\delta}{10C_{\max}}$,
$1\leqslant v < v_{\max}$ inequality
\beq \label{21.10.3}
x_{\tau_{2v}}^{\varepsilon}\leqslant m+ \varepsilon C_{\max}+\varepsilon\xi_{\tau_{2v}}
\leqslant m+ \varepsilon C_{\max}+\frac{2\delta}{5}<m+\frac{\delta}{2}
\leqslant x+\frac{\delta}{2}
\eeq
holds.

Now we use proof by induction. Consider the case when
 $\tau_{2v-2}<k< \tau_{2v-1}$, $1\leqslant v < v_{\max}$. Since $f'(x_{\tau_{2v-2}}^{\varepsilon})>0$ it follows from definition of stopping time $\tau_0$
and formulae (\ref{21.10.2}), (\ref{21.10.3})  that on event $A_{\frac{\delta}{5}}\left(\varepsilon\right)$ the following inequality holds for $k=\tau_{2v-2}+1$, $1\leqslant v < v_{\max}$:
\beq \label{21.10.5}
x_k^{\varepsilon}=x_{\tau_{2v-2}}^{\varepsilon}-\varepsilon f'(x_{\tau_{2v-2}}^{\varepsilon})+\varepsilon\xi_{\tau_{2v-2}+1}
\leqslant x+\frac{\delta}{2}+\frac{2\delta}{5}<x+\delta.
\eeq

We fix $1\leqslant v < v_{\max}$.
We assume that for some $\tau_{2v-2}+1<l<\tau_{2v-1}$ inequality (\ref{21.10.5}) holds for every $\tau_{2v-2}+1\leqslant k\leqslant l$. Let us show that it holds for $k=l+1$. It follows from induction hypothesis that on event $A_{\frac{\delta}{5}}\left(\varepsilon\right)$ we have
\beq \label{21.10.7}
f'(x_k^{\varepsilon})>0
\eeq
for every $\tau_{2v-2}+1\leqslant k\leqslant l$.

From inequalities (\ref{21.10.2}), (\ref{21.10.3}), (\ref{21.10.7}) on event $A_{\frac{\delta}{5}}\left(\varepsilon\right)$ we have
$$
x_{l+1}^{\varepsilon}=x_{\tau_{2v-2}}^{\varepsilon}-\varepsilon \sum\limits_{j=\tau_{2v-2}}^{l}f'(x_j^{\varepsilon})
+\varepsilon\sum\limits_{j=\tau_{2v-2}+1}^{l+1}\xi_j
\leqslant x_{\tau_{2v-2}}^{\varepsilon}+\frac{2\delta}{5}<x+\delta.
$$
Hence we proved that induction hypothesis is true, i.e., on event $A_{\frac{\delta}{5}}\left(\varepsilon\right)$ inequality (\ref{21.10.5})
holds for every $\tau_{2v-2}+1\leqslant k< \tau_{2v-1}$ and every $1\leqslant v < v_{\max}$.

Lemma \ref{l.1} follows from inequalities (\ref{21.10.1}), (\ref{21.10.3})
and the fact that formula (\ref{21.10.5})
holds for every $\tau_{2v-2}+1\leqslant k< \tau_{2v-1}$, $1\leqslant v < v_{\max}$.
\end{proof}

Let us define for every $h>0$
$$
C_{x,h}\left(\varepsilon\right):=\left\{\min\limits_{0\leqslant k \leqslant n_\varepsilon}x_k^{\varepsilon}>x-h\right\}.
$$

\begin{lemma} \label{l.2}
For every $0<\delta<m-M_{l}$, $x\in(M_{l}+\delta, m]$,
$0<\varepsilon<\frac{\delta}{10 C_{\max}}$ the following equality holds:
$$
A_{\frac{\delta}{5}}\left(\varepsilon\right)\cap C_{x,\delta}\left(\varepsilon\right)=A_{\frac{\delta}{5}}\left(\varepsilon\right).
$$
\end{lemma}
\begin{proof}
See proof of Lemma \ref{l.1}.
\end{proof}

\begin{lemma} \label{l.3}
Let $\lim\limits_{\varepsilon\downarrow 0}\varepsilon n_\varepsilon=\infty$ and let $x_0 \in (M_{l}+\Delta, M_{r}-\Delta)$ for some $\Delta >0$.
% \begin{enumerate}
%     \item $x\in(M_{r-1}+\Delta,m_r]$, где $\Delta\in(0,m_r-M_{r-1})$;\label{l.3.1}
%     \item $x\in[m_r, M_{r}-\Delta)$, где $\Delta\in(0,M_{r}-m_r)$.\label{l.3.2}
% \end{enumerate}
Then for every $\delta>0$ there exists $\varepsilon_0 > 0$ such that for every $0<\varepsilon<\varepsilon_0$ on event $A_{\frac{\delta}{5}}\left(\varepsilon\right)$ the following holds:
\beq\label{20.11.1}
|x_{\lfloor n_\varepsilon \rfloor}^{\varepsilon} - m| < 2\delta.
\eeq
\end{lemma}
\begin{proof}[Proof of Lemma~\ref{l.3}]
Let assume that 
\beq\label{l.3.1}x\in(M_{l}+\Delta,m].
\eeq  
For sufficiently small $\delta>0$ we define stopping times
$$
\tau_\delta=\inf\{0 \leqslant k \leqslant n_\varepsilon:x_k^{\varepsilon}\in (m-\delta,m+\delta)\}.
$$
For $\varepsilon>0$ let
$$
z_\varepsilon(\delta):=\left\lceil\frac{m-x+\frac{\delta}{5}}{\varepsilon c_{\min}}\right\rceil.
$$
Here $c_{\min}:=\min\limits_{y\in[x-\delta,m - \delta]}|f'(y)|$.

We show that for every $0<\varepsilon<\frac{\delta}{10 C_{\max}}$ on event $A_{\frac{\delta}{5}}\left(\varepsilon\right)$
inequality $\tau_\delta\leqslant z_\varepsilon(\delta)$ holds. If $x = m$, then it is obviously satisfied. Let $ x < m$. Then there exists $\delta > 0$ such that $x \in [x - \delta, m - \delta]$. Assume that $\tau_\delta > z_\varepsilon(\delta)$.
Since we cannot overshoot interval $(m-\delta,m+\delta)$,
inequality
$$
x_k^{\varepsilon}\leqslant m-\delta
$$
must hold for every $1\leqslant k \leqslant z_\varepsilon(\delta)$.

It follows from Lemma \ref{l.2} that for every $0<\varepsilon<\frac{\delta}{10 C_{\max}}$ we have
$$
x_k\geqslant x-\delta,\ \ \ 1\leqslant k \leqslant z_\varepsilon(\delta).
$$
Therefore, for every $1\leqslant k \leqslant z_\varepsilon(\delta)$ on event $A_{\frac{\delta}{5}}\left(\varepsilon\right)$ we have $x_k^{\varepsilon}\in[x-\delta,m-\delta]$. However, this implies that on the event $ A_{\frac{\delta}{5}}$
we have
$$
x_{z_\varepsilon(\delta)}^{\varepsilon} \geqslant x +\varepsilon \sum\limits_{k=1}^{z_\varepsilon(\delta)} c_{\min}-\frac{\delta}{5} \geqslant m.
$$
We have a contradiction. Hence, for every $0<\varepsilon<\frac{\delta}{10 C_{\max}}$ on event $A_{\frac{\delta}{5}}\left(\varepsilon\right)$ we have $\tau_\delta\leqslant z_\varepsilon(\delta)$. Since $z_\varepsilon(\delta) = o\left(n_{\varepsilon}\right)$ there exists $0<\hat{\varepsilon}<\frac{\delta}{10 C_{\max}}$ such that $z_\varepsilon(\delta) < n_{\varepsilon}$ for every $0<\varepsilon<\hat{\varepsilon}$. Thus, for every $0<\varepsilon<\hat{\varepsilon}$ on event $A_{\frac{\delta}{5}}\left(\varepsilon\right)$ we have
\beq\label{24.10.1}
\tau_\delta < n_\varepsilon.
\eeq
Let us assume that at time moment $\tau_\delta$ the process reached interval $[m,m+\delta)$. Then from Lemma \ref{l.1} for every $\tau_\delta \leqslant k \leqslant n_{\varepsilon}$ the process will lie in interval $(\infty,m+2\delta)$. Given that on event $A_{\frac{\delta}{5}}\left(\varepsilon\right)$ the jumps do not exceed $\varepsilon C_{\max}+\frac{2\delta}{5}$ and we can choose sufficiently small $0 < \varepsilon_0 < \hat{\varepsilon}$ such that for every $0 < \varepsilon < \varepsilon_0$ the process cannot jump for domain to the right of $m$ to a point on the left without visiting $(m-\delta,m]$. Therefore, from Lemma \ref{l.2} the process will lie in interval $(m-2\delta,\infty)$ until moment $n_{\varepsilon}$. Thus, the process will not leave set
$(-\infty,m+2\delta)\cap(m-2\delta,\infty)=(m-2\delta,m+2\delta)$ until moment $n_{\varepsilon}$. This concludes the proof of the theorem.
\end{proof}

\section{Suitable Time Scaling}\label{sec_ap_3}
\begin{proof}[Proof of Theorem~\ref{t.1}]
Let us prove \ref{1_t.1}. For every $\delta > 0$ for sufficiently small $\varepsilon$ we have
    $$
    \begin{aligned}
        & \p{\left| x_{\lfloor n_\varepsilon \rfloor}^{\varepsilon} - m \right| > \delta} = \p{\left| x_{\lfloor n_\varepsilon \rfloor}^{\varepsilon} - m \right| > \delta, A_{\frac{\delta}{10}}\left(\varepsilon\right)} + \p{\left| x_{\lfloor n_\varepsilon \rfloor}^{\varepsilon} - m \right| > \delta, \overline{A}_{\frac{\delta}{10}}\left(\varepsilon\right)} \overset{\text{L.\ref{l.3}}}{=} \\
        & \p{\left| x_{\lfloor n_\varepsilon \rfloor}^{\varepsilon} - m \right| > \delta, \overline{A}_{\frac{\delta}{10}}\left(\varepsilon\right)} \leqslant \p{\varepsilon\max\limits_{1\leqslant k \leqslant \lfloor n_\varepsilon \rfloor}\left|\sum\limits_{j=1}^k\xi_j\right| \geqslant \frac{\delta}{10}} \overset{\text{L.\ref{l.03}}}{\longrightarrow} 0 \text{ for } \varepsilon \downarrow 0.
    \end{aligned}
    $$
Let us prove \ref{2_t.1}.
% Докажем пункт \ref{2_t.1}. Определим $n_\varepsilon$ из следующего соотношения
% $$
% \lim_{\varepsilon \downarrow 0} \frac{n_\varepsilon}{\left(\frac{1}{\varepsilon}\right)^\alpha L^\alpha\left(\frac{1}{\varepsilon}\right)} < \infty.
% $$
Monotonicity of $\overline{n}_\varepsilon$ follows from condition~\ref{22.10.1}. Relation $\lim\limits_{\varepsilon\downarrow 0}H\left(1/\varepsilon\right) \overline{n}_\varepsilon = 0$ holds because of Lemma \ref{l.01}. Relation $\lim\limits_{\varepsilon\downarrow 0}\varepsilon \overline{n}_\varepsilon=\infty$ holds because of properties of slowly varying functions. Next, we will use proof by contradiction.

Let us assume that there exist
$\tilde{\delta}>0$, $\gamma>0$ such that for event
$$
B_{\tilde{\delta}}:=\left\{\limsup\limits_{\varepsilon\downarrow 0}|x_{\lfloor n_\varepsilon \rfloor}^{\varepsilon}-m|\geq\tilde{\delta}\right\}
$$
the following inequality holds:
$$
\p{B_{\tilde{\delta}}}\geq\gamma.
$$
Consider event 
$$
G_{\frac{\tilde{\delta}}{10}}\left(\epsilon\right) = \bigcap_{0<\varepsilon < \epsilon} A_{\frac{\tilde{\delta}}{10}}\left(\varepsilon\right).
$$
It follows from Corollary \ref{c.1} that for every $\tilde{\delta} > 0$ we have
$$
\lim_{\varepsilon \downarrow 0} \mathbbm{1}\left(A_{\frac{\tilde{\delta}}{10}}\left(\varepsilon\right)\right) = 1 \text{ a.s.}.
$$
Therefore,
\beq\label{20.11.2}
\lim_{\epsilon \downarrow 0} \mathbbm{1}\left(G_{\frac{\tilde{\delta}}{10}}\left(\epsilon\right)\right) = 1 \text{ a.s.}.
\eeq
It follows from formula (\ref{20.11.2}) that for every $\gamma > 0$ there exists  $\epsilon_\gamma>0$ such that for every $0<\epsilon<\epsilon_\gamma$ we have
$$
\p{G_{\frac{\tilde{\delta}}{10}}\left(\epsilon\right)}\geqslant 1-\frac{\gamma}{2}.
$$
Thus, for $0<\epsilon<\epsilon_\gamma$
we have
$$
\p{G_{\frac{\tilde{\delta}}{10}}\left(\epsilon\right)\cap B_{\tilde{\delta}}}\geq\frac{\gamma}{2}.
$$
From Lemma \ref{l.3} on event $G_{\frac{\tilde{\delta}}{10}}\left(\epsilon\right)\cap B_{\tilde{\delta}}$
for every $0<\varepsilon<\epsilon\wedge\varepsilon_0$ the following holds:
$$
|x_{\lfloor n_\varepsilon \rfloor}^{\varepsilon}-m|<\tilde{\delta},
$$
that contradicts the definition of event $B_{\tilde{\delta}}$. This concludes the proof. 
\end{proof}

\begin{proof}[Proof of Theorem~\ref{t.2}]
The proof is completely analogous to the proof of Theorem \ref{t.1}. So, we omit it from this paper.
\end{proof}

\subsection{Numerical Illustration of Theorems~\ref{t.1} and~\ref{t.2}}\label{sec_additional_sim_for_part_A}

To illustrate Theorems \ref{t.1} and \ref{t.2}, we ran SGD on a one-dimensional cubic-spline target with two wells at $x=-1$ and $x=1$ separated by a peak at $x=0$. We fixed the learning rate at $\varepsilon=10^{-4}$ and compared two noise regimes: (i) centered $\alpha$-stable noise with $\alpha=1.5$ (Hypothesis \ref{H1}), and (ii) standard Gaussian noise (Hypothesis \ref{H2}). On time scales comparable to $n_{\varepsilon}$, the process tends to concentrate near one of the wells. Over sufficiently long time scales, both noise models will eventually induce jumps between the two minima. However, inter-well oscillations are significantly more frequent and pronounced under heavy-tailed $\alpha$-stable perturbations compared to Gaussian noise. Indeed, in our simulations, Gaussian noise did not yield any transitions within the considered time frame; observing such jumps would require simulations extended by orders of magnitude. This difference between the two noise types is illustrated in Figure~\ref{fig:oscilation}.
\begin{figure}[ht]
  \centering
  \begin{subfigure}[t]{0.45\textwidth}
    \includegraphics[width=\linewidth]{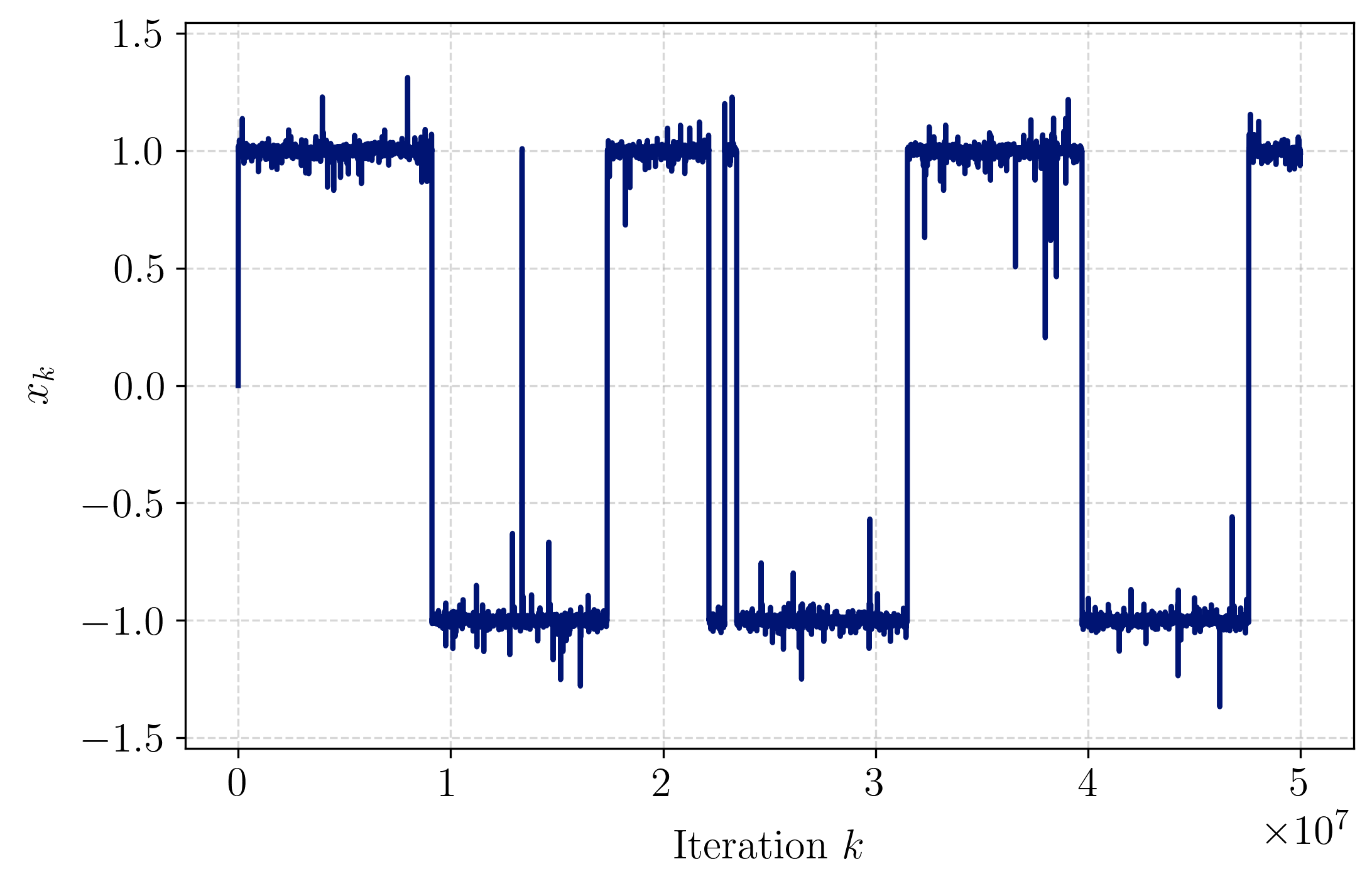}
    \caption{Trajectory of SGD with $\alpha$-stable noise ($\alpha=1.5$, $\beta = 0$)}
    \label{fig:image1}
  \end{subfigure}\hfill
  \begin{subfigure}[t]{0.45\textwidth}
    \includegraphics[width=\linewidth]{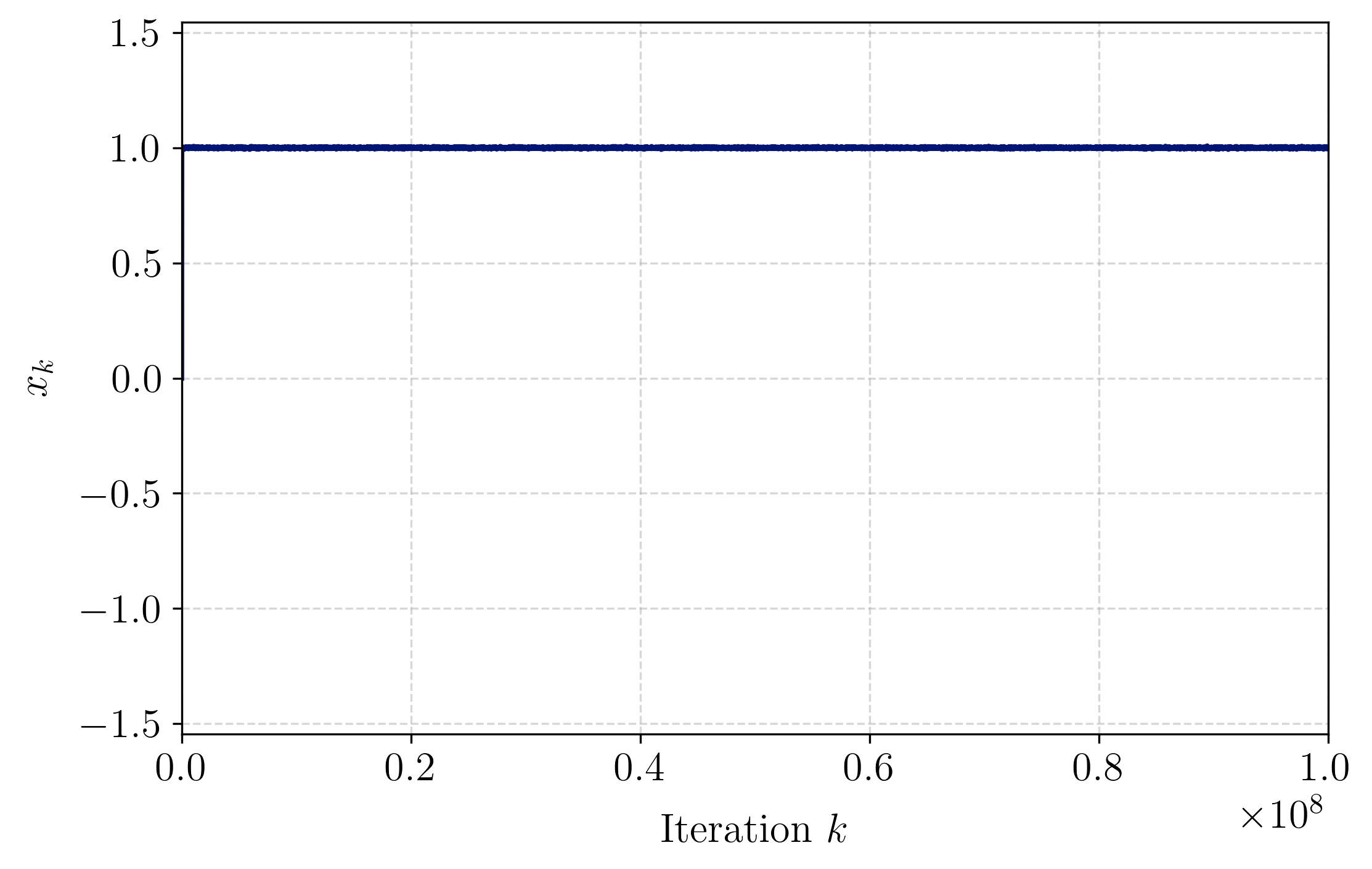}
    \caption{Trajectory of SGD with Gaussian noise}
    \label{fig:image2}
  \end{subfigure}
  \caption{SGD trajectories on a double-well cubic-spline potential under different noise types.}
  \label{fig:oscilation}
\end{figure}
Table \ref{tab:neps} summarizes the estimated convergence thresholds $n_{\varepsilon}$ for different values of $\varepsilon$, as well as the fraction of SGD trajectories that end within the band $(0.75, 1.25)$ after $n_{\varepsilon}$ steps. For each setting, the fraction was estimated based on $1000$ independent runs, with initial points sampled uniformly from the interval $(0.0, 1.9)$. The results show that, in both noise regimes, SGD eventually concentrates around the local minimum, but the required theoretical number of steps differs drastically between the $\alpha$-stable and Gaussian cases.
\begin{table}[ht]
\centering
\small
\begin{tabular}{%
  l  % group label
  c  % ε
  c  % n_ε
  c  % %×n_ε∈(0.75,1.25)
}
\toprule
& $\varepsilon$ 
& $n_{\varepsilon}$ 
& Fraction of $x_{n_\varepsilon}$ within $(0.75,1.25)$ \\
\midrule
\multicolumn{4}{l}{\textbf{$\alpha$-stable}}\\[-0.3em]
  & $0.1$     & $76$            & $0.36$ \\
  & $0.01$    & $1\,892$        & $0.701$ \\
  & $0.001$   & $47\,514$       & $0.775$ \\
  & $0.0001$  & $1\,193\,478$   & $0.789$ \\
  & $1\times10^{-5}$   & $29\,978\,812$  & $0.835$ \\
\addlinespace
\multicolumn{4}{l}{\textbf{Normal}}\\[-0.3em]
  & $0.1$     & $79$            & $0.939$ \\
  & $0.01$    & $6\,309$        & $1.00$  \\
  & $0.001$   & $501\,187$      & $1.00$  \\
  & $0.0001$  & $39\,810\,717$  & –       \\
  & $1\times10^{-5}$   & $3\,162\,277\,660$ & –  \\
\bottomrule
\end{tabular}
\caption{Estimated sample sizes $n_{\varepsilon}$ and the fraction of SGD iterates falling within $(0.75, 1.25)$ after $n_{\varepsilon}$ steps. Each estimate is based on 1000 simulated trajectories from uniformly random initial points in $(0.0, 1.9)$.}
\label{tab:neps}
\end{table}
To explore behavior beyond the convergence threshold, we extended our Levy-stable simulations to ten times the steps $n_{\varepsilon}$ predicted by Theorem \ref{t.1}. On these much longer time scales, the process under heavy-tailed noise no longer remains confined near the closest well, and indeed the fraction of iterates within the band $(0.75,1.25)$ stops improving and oscillates around $0.5$.
\begin{table}[ht]
  \centering
  \small
  \begin{tabular}{%
    l  % group label
    c  % ε
    c  % extended steps
    c  % fraction in band
  }
  \toprule
  & $\varepsilon$ 
  & $10 n_{\varepsilon}$
  & Fraction of $x_{n_\varepsilon}$ within $(0.75,1.25)$ \\
  \midrule
  \textbf{$\alpha$–stable}\\[-0.3em]
    & $0.1$     & $760$           & $0.029$ \\
    & $0.01$    & $18,920$        & $0.439$ \\
    & $0.001$   & $475,140$       & $0.508$ \\
    & $0.0001$  & $11,934,780$    & $0.487$ \\
  \bottomrule
  \end{tabular}
  \caption{Estimated sample sizes $10 n_{\varepsilon}$ and the fraction of SGD iterates falling within $(0.75, 1.25)$ after $10 n_{\varepsilon}$ steps. Each estimate is based on 1000 simulated trajectories from uniformly random initial points in $(0.0, 1.9)$.}
  \label{tab:levy_extended}
\end{table}

\section{Sticking to a Critical Point}\label{sec_ap_4}

As mentioned in the main body of this article, in this section we first present precise formulations of the results followed by their proofs. For the reader's convenience, we reiterate the setup and thereby render this section self-contained.

\subsection{Sticking to a Critical Point: precise formulation} 

We establish conditions under which the SGD sequence remains in the vicinity of a critical point $c \in \mathbb{R}$ for an extended period. The point $c$ can be either an extremum or a inflection point. We  require assumptions about the function to hold within a neighborhood of $c \in \mathbb{R}$. Specifically, suppose there is some $K \geqslant 1$ such that $c \in \mathbb{R}$ is a $K$-critical point of $f$. In other words, for all $k = 1, \ldots, K$ we have $f^{(k)}(c) = 0$ and also $f^{(K+1)}(c) \neq 0$. Furthermore, suppose there is a $\Delta > 0$ so that
\[
\sup_{c - \Delta \leqslant x \leqslant c + \Delta} |f^{(K+1)}(x)| < \infty.
\]
To simplify our analysis, we impose condition, which ensures that the asymptotic inverse function of $L(x)$ can be expressed in terms of $L(x)$ itself:
\begin{equation}\label{crit_point_1}
    L(x) \sim L(xL(x)), \quad x \ra \infty. 
\end{equation}
If condition \eqref{crit_point_1} is not satisfied, then in Theorem~\ref{th1_crit_point} (and Theorem~\ref{th2_crit_point}) one can replace $\delta(\varepsilon)$ by $\varepsilon^{a}$, where $0 <a <  \frac{\alpha-1}{K-1+\alpha} \left(\frac{1}{K+1}\right)$ and the statements will still hold.

Under these conditions, the SGD sequence remains within a shrinking neighborhood of $c$. More precisely, the following two theorems apply.
\begin{theorem}\label{th1_crit_point}
    Let \ref{H1} hold and let $L(u) > 0$ be a s.v.f., satisfying \eqref{crit_point_1} and $\e{|\xi_1|^\alpha L^\alpha(|\xi_1|)}<\infty$. For any $o_{\varepsilon}\ra 0$ as $\varepsilon\downarrow 0$, define
    \begin{equation*}
        \delta(\varepsilon) = %\frac{1}{o_{\varepsilon}} 
        \varepsilon^{\frac{\alpha-1}{K-1+\alpha}} L^{\frac{-1}{K-1+\alpha}}\left(o_{\varepsilon}\,\varepsilon^{\frac{-K}{K-1+\alpha}} \right).
    \end{equation*}
    Then, for all $t >0$ and as $\varepsilon\downarrow 0$,
    \begin{equation*}
        \sup\limits_{n \leqslant t h(\varepsilon)}|x^{\varepsilon}_{n} - c| \leqslant \delta(\varepsilon) \text{ almost surely (a.s.) }
    \end{equation*}
    holds uniformly over all $x_0 \,:\, |x_0 - c|< \frac{1}{3} \delta(\varepsilon)$. Here, 
    \begin{equation*}
         h(\varepsilon):= o^{\alpha}_{\varepsilon} \varepsilon^{ - 1} \delta^{- (K-1)}(\varepsilon) \text{ and we have } h(\varepsilon) = o(\varepsilon^{-\gamma}) \text{ for }  \gamma >\frac{\alpha K}{K-1+\alpha}.
         % \varepsilon^{-\frac{(K-2)\alpha}{K-1 + \alpha}},  \quad o_{\varepsilon}(1):= (\varepsilon h(\varepsilon))^{-\frac{1}{K-1}}.
    \end{equation*}
\end{theorem}

\begin{theorem}\label{th2_crit_point}
    Let \ref{H2} hold and let $L(u) > 0$ be a s.v.f., satisfying \eqref{crit_point_1} holds and $L(u) = o\left( (\ln \ln u)^{-\frac{1}{2}}\right)$.  Define
    \begin{equation*}
        \delta(\varepsilon) = 
        \varepsilon^{\frac{1}{K+1}} L^{\frac{-2}{K +1}}\left(\varepsilon^{\frac{-K}{K+1}} \right).
    \end{equation*}
    Then, for all $t >0$ and as $\varepsilon\downarrow 0$ hold
    \begin{equation*}
        \sup\limits_{n \leqslant t h(\varepsilon)}|x^{\varepsilon}_{n} - c| \leqslant \delta(\varepsilon) \text{ a.s. }
    \end{equation*}
    holds uniformly over all $x_0 \,:\, |x_0 - c|< \frac{1}{3} \delta(\varepsilon)$. For any $o(1) \ra 0$,
    \begin{equation*}
         h(\varepsilon):= o\left(\varepsilon^{ - 1} \delta^{- (K-1)}(\varepsilon)\right)  \text{ and we have } h(\varepsilon) = o\left(   \varepsilon^{-\gamma}\right) \text{ for }   \gamma > \frac{2K}{K+1}.
         % \varepsilon^{-\frac{(K-2)\alpha}{K-1 + \alpha}},  \quad o_{\varepsilon}(1):= (\varepsilon h(\varepsilon))^{-\frac{1}{K-1}}.
    \end{equation*}
\end{theorem}

\subsection{Sticking to a Critical Point: proofs }

Under the condition \eqref{crit_point_1} for all $s \in \mathbb{R}$ is satisfied
\begin{equation}\label{crit_point_2}
    L(x) \sim L(xL^s(x)), \quad x \ra \infty. 
\end{equation}

\begin{proof}[Proof of Theorem~\ref{th1_crit_point}]
We use Taylor expansion for $f'(x_n)$ in a neighborhood of point $c:$
\begin{equation*}
    f'(x_n) = \frac{f^{(K+1)}(\widehat{x}_n)}{K!} (x_n-c)^K.
\end{equation*}
Here $\widehat{x}_n  = c + \theta (x_n -c)$, $\theta \in (0,1).$ Then for $n \leqslant t h(\varepsilon)$  SGD process satisfies relation 
\begin{equation}\label{eq_crit_esc_1}
\begin{split}
    |x_{n+1} - c| & = \left| (x_0 - c) + \varepsilon \sum_{m=0}^n \xi_m  - \varepsilon \sum_{m=0}^n \frac{f^{(K+1)}(\widehat{x}_m)}{K!} (x_m-c)^K\right| \\
    & \leqslant |x_0 - c| + \varepsilon 
    \sup_{n\leqslant t\, h(\varepsilon)}\left|\sum_{m=0}^n \xi_m \right| + C t\,\varepsilon h(\varepsilon)\,\sup\limits_{m \leqslant n} |x_m - c |^K.
\end{split}
\end{equation}
Next, it is easy to see that for $\varepsilon \downarrow 0$ functions $\delta(\varepsilon)$ and $ h(\varepsilon)$ satisfy relation
\begin{equation*}
    \varepsilon h(\varepsilon) \delta^{K-1}(\varepsilon) = o_{\varepsilon}^{\alpha} \ra 0, 
\end{equation*}
i.e., for sufficiently small $\varepsilon$ we have 
\begin{equation}\label{eq_crit_esc_2}
   Ct\,\varepsilon h(\varepsilon) \delta^{K}(\varepsilon)  \leqslant \frac{1}{3} \delta(\varepsilon).
\end{equation}
In the remainder of the proof $o_{\varepsilon}(1)$ can denote perhaps different $o(1) \ra 0.$
From Lemma~\ref{anh2021_t31} for some $o_{\varepsilon}(1) \ra 0$ we have
\begin{equation*}
    \varepsilon  \sup_{n\leqslant t\, h(\varepsilon)}\left|\sum_{m=0}^n \xi_m \right|   = \varepsilon \, o_{\varepsilon}(1) b_{th(\varepsilon)} = \varepsilon \, o_{\varepsilon}(1) (t h(\varepsilon))^{\frac{1}{\alpha}} \widetilde{L} ^{\frac{1}{\alpha}} \left( (t h(\varepsilon))^{\frac{1}{\alpha}}\right).
\end{equation*}
From definitions of $h(\varepsilon), \delta(\varepsilon)$ we have
\begin{equation*}
    \varepsilon h^{\frac{1}{\alpha}}(\varepsilon) = o_{\varepsilon}\left( 1\right) \delta(\varepsilon) L^{\frac{1}{\alpha}}\left(o_{\varepsilon}\,\varepsilon^{\frac{-K}{K-1+\alpha}}\right).
\end{equation*}
Therefore, 
\begin{equation}\label{eq_crit_esc_3}
\begin{split}
    \varepsilon  \sup_{n\leqslant t\, h(\varepsilon)}\left|\sum_{m=0}^n \xi_m \right| &  =  o_{\varepsilon}\left( 1\right) \delta(\varepsilon)  L^{\frac{1}{\alpha}}\left(o_{\varepsilon}\,\varepsilon^{\frac{-K}{K-1+\alpha}}\right) \widetilde{L} ^{\frac{1}{\alpha}} \left( o_{\varepsilon} t^{\frac{1}{\alpha}} \varepsilon^{\frac{-K}{K-1+\alpha}} L^{\frac{K-1}{\alpha(K-1+\alpha)}}  ( o_{\varepsilon}\,\varepsilon^{\frac{-K}{K-1+\alpha}})\right) \\ 
    & \leqslant \frac{1}{3} \delta(\varepsilon).
\end{split}
\end{equation}
Here the last inequality follows from \eqref{crit_point_2}.
Thus, for any $t>0$ and sufficiently small $\varepsilon$ we use inequalities $|x_0 - c|< \frac{1}{3} \delta(\varepsilon)$ and \eqref{eq_crit_esc_1}--\eqref{eq_crit_esc_3} to get recursion for $n \leqslant t h(\varepsilon):$
\begin{equation*}
    \text{ if }\sup\limits_{m \leqslant n} |x_m - c | < \delta(\varepsilon), \text{ then } \sup\limits_{m \leqslant n+1} |x_m - c | < \delta(\varepsilon).
\end{equation*}
The proof of the theorem is complete.
\end{proof}

\begin{proof}[Proof of Theorem~\ref{th2_crit_point}]
Similarly to the proof of Theorem~\ref{th1_crit_point} we get recursion with regards to $\sup\limits_{m \leqslant n} |x_m - c | < \delta(\varepsilon)$
for $n \leqslant t h(\varepsilon).$ The only difference lies in the usage of Lemma~\ref{l.02} instead of Lemma~\ref{anh2021_t31} to prove relations
\begin{equation}\label{eq_crit_esc_4}
    \frac{\varepsilon}{\delta(\varepsilon)}  \sup_{n\leqslant t\, h(\varepsilon)}\left|\sum_{m=0}^n \xi_m \right|  \ra 0  \text{ a.s..}
\end{equation}
Given Lemma~\ref{l.02}, it is sufficient for proving \eqref{eq_crit_esc_4} to show that 
\begin{equation}\label{crit_point_3}
t h(\varepsilon) < \check{n}_{\frac{\varepsilon}{\delta(\varepsilon)}}=\left(\frac{\delta(\varepsilon)}{\varepsilon}\right)^2L^2\left(\frac{\delta(\varepsilon)}{\varepsilon}\right).
\end{equation}
Inequality \eqref{crit_point_3} is equivalent to
\begin{equation*}
    t\,o\!\left(\varepsilon\right) < \delta^{K+1}(\varepsilon) L^{2}\left(\frac{\delta(\varepsilon)}{\varepsilon}\right) = \varepsilon L^{-2}\left(\varepsilon^{-\frac{K}{K+1}}\right) L^{2}\left(\varepsilon^{-\frac{K}{K+1}} L^{\frac{-2}{K+1}}\left(\varepsilon^{-\frac{K}{K+1}}\right) \right) \sim \varepsilon.
\end{equation*}
Here the last equivalence follows from \eqref{crit_point_2}. The proof of the theorem is complete.
\end{proof}

\section{Leaving Neighborhood of a Sharp Maximum}\label{sec_ap_5}

\begin{lemma}\label{bound_prob_lemma}
The random variable $N^{\varepsilon}$ is a proper random variable for all $\varepsilon > 0$ and
$$
N^{\eps} \pra \infty \text{ as } \varepsilon \downarrow 0.
$$
\end{lemma}
\begin{proof}
Let's first prove that $\p{N^{\varepsilon} = \infty} = 0$. We have
\begin{equation*}
\begin{split}
    \p{N^{\varepsilon} = \infty} & = \p{\left|x_0^{\varepsilon}\right| \leqslant \delta, \left|x_1^{\varepsilon}\right| \leqslant \delta, \left|x_2^{\varepsilon}\right| \leqslant \delta, \ldots} \\
    &  = \p{\bigcap_{k=0}^{\infty} \skk{ \sup\limits_{{\tau_{2 k}^{\shortuparrow}(\varepsilon)}\leqslant j <   \tau_{2 k+1}^{\shortdownarrow}(\varepsilon)} x_j^{\varepsilon} < \delta }  \cap \skk{\inf\limits_{{\tau_{2 k+1}^{\shortdownarrow}(\varepsilon)}\leqslant j <   \tau_{2 k+2}^{\shortuparrow}(\varepsilon)} x_j^{\varepsilon} > -\delta} }
\end{split}
\end{equation*}
Since the jump through zero has a certain sign $x_{\tau_{2 k}^{\shortuparrow}(\varepsilon)}^{\varepsilon} \geqslant 0,  x_{\tau_{2 k+1}^{\shortdownarrow}(\varepsilon)}^{\varepsilon} < 0, $  then 
\begin{equation*}
\begin{split}
\skk{ \sup\limits_{{\tau_{2 k}^{\shortuparrow}(\varepsilon)}\leqslant j <   \tau_{2 k+1}^{\shortdownarrow}(\varepsilon)} x_j^{\varepsilon}  < \delta } & \subseteq  \skk{ \sup\limits_{{\tau_{2 k}^{\shortuparrow}(\varepsilon) + 1}\leqslant j <   \tau_{2 k+1}^{\shortdownarrow}(\varepsilon)} \varepsilon\left( \sum_{n= \tau_{2 k}^{\shortuparrow}(\varepsilon) + 1}^j (\xi_n + c_r)\right) < \delta} := A_{2k},  \\ 
\skk{ \inf\limits_{{\tau_{2 k+1}^{\shortdownarrow}(\varepsilon)} \leqslant j <   \tau_{2 k+2}^{\shortuparrow}(\varepsilon)} x_j^{\varepsilon}  > -\delta }& \subseteq \skk{ \inf\limits_{{\tau_{2 k+1}^{\shortdownarrow}(\varepsilon)} + 1\leqslant j <   \tau_{2 k+2}^{\shortuparrow}(\varepsilon)}  \varepsilon\left( \sum_{n= \tau_{2 k+1}^{\shortdownarrow}(\varepsilon) + 1}^j (\xi_n - c_l)\right) > -\delta} := A_{2k+1}, 
\end{split}
\end{equation*}
and for all $\varepsilon > 0,$ by virtue of positivity of $c_r, c_l >0,$ we have
\begin{equation*}
    \p{A_{2k}} \leqslant \p{ \sup\limits_{1\leqslant j <   \tau_{1}^{\shortdownarrow}} \left( \sum_{n= 1}^j (\xi_n + c_r)\right) < \frac{\delta}{\varepsilon} } := p_{1}(\varepsilon) < 1, 
\end{equation*}
\begin{equation*}
    \p{A_{2k+1}} \leqslant  \p{ \inf\limits_{1\leqslant j <   \tau_{1}^{\shortuparrow}} \left( \sum_{n= 1}^j (\xi_n - c_l)\right) > -\frac{\delta}{\varepsilon}} := p_{2}(\varepsilon) < 1, 
\end{equation*}
where $\tau_{1}^{\shortdownarrow}:=  \inf \left\{j> 0 : \sum_{n= 1}^j (\xi_n + c_r) <0\right\}, \tau_{1}^{\shortuparrow}:= \inf \left\{j> 0 : \sum_{n= 1}^j (\xi_n - c_l) \geqslant 0\right\}.$
Hence
\begin{equation*}
     \p{N^{\varepsilon} = \infty} \leqslant \p{\bigcap_{k=0}^{\infty} A_k }  = \prod_{k=0}^{\infty}\p{A_k } \leqslant 
 \prod_{k=0}^{\infty} \, p_{1}(\varepsilon) p_{2}(\varepsilon) = 0. 
\end{equation*}

Now we prove that $N^{\eps} \pra \infty, \text{ as } \varepsilon \downarrow 0$. For all fixed $k \geqslant 1$ for sufficiently large enough $\varepsilon>0$ from the equation \eqref{eq_sharp_max_2}   we have
\begin{equation*}
\begin{split}
\p{N^{\varepsilon} > k} =  &  \p{\left|x_0^{\varepsilon}\right| \leqslant \delta, \left|x_1^{\varepsilon}\right| \leqslant \delta, \ldots, \left|x_k^{\varepsilon}\right| \leqslant \delta} \\ 
& \geqslant  \p{\left|x_0^{\varepsilon}\right| \leqslant \frac{\delta}{2}, \left|\xi_1 - f'(x_0^{\varepsilon})\right| \leqslant \frac{\delta}{2\varepsilon}, \ldots, \left| \sum_{i=0}^{k-1} \xi_{i+1} - f'(x_i^{\varepsilon})\right| \leqslant \frac{\delta}{2\varepsilon} }\\
& \geqslant  \p{\left|x_0^{\varepsilon}\right| \leqslant \frac{\delta}{2}, \left|\xi_1\right| \leqslant \frac{\delta}{4\varepsilon}, \ldots, \left| \sum_{i=0}^{k-1} \xi_{i+1} \right| \leqslant \frac{\delta}{4\varepsilon} } \ra 1, \text{ as } \varepsilon \downarrow 0.
\end{split} 
\end{equation*}
\end{proof}

Let's prove Remark~\ref{rk_sharp_max} first, and then Theorem~\ref{th_sharp_max}.
\begin{proof}[Proof of Remark~\ref{rk_sharp_max}]
\ref{rk_sm_i}.
If $\tau_1^{\shortdownarrow}<\infty$,  then until the moment $\tau_1^{\shortdownarrow}$ the jumps of the RRW $\left\{X_n\right\}$ coincide in distribution with $\xi_n + c_r$ and
$$ 
\p{\tau_1^{\shortdownarrow}<\infty}= \p{\inf\limits_{n \geqslant 1} \skk{\sum\limits_{k=1}^n \left(\xi_k +c_r\right)} < 0} = p_{\shortdownarrow}.
$$
For trajectories satisfying the condition $\tau_1^{\shortdownarrow}<\infty$, we will fix the value $X_{\tau_1^{\shortdownarrow}}=-x<0$ and further consider the subsequent part of the trajectory starting from the point $-x$. At the event $\left\{\tau_2^{\shortdownarrow}<\infty\right\}$, the jumps of this trajectory will have the same distribution from $\xi_1 - c_l$ until the moment $\tau_2^{\shortdownarrow},$ and due to the positivity of $ -X_{\tau_1^{\shortdownarrow}} = x > 0$ we have
\begin{equation}\label{eq_loc_new_1}
\begin{split} 
\p{ \tau_2^{\shortuparrow}<\infty } & =  \p{\tau_1^{\shortdownarrow}<\infty, \tau_2^{\shortuparrow}<\infty } = \p{\tau_1^{\shortdownarrow}<\infty,  \ \sup\limits_{n\geqslant 1} \skk{\sum\limits_{k=\tau_1^{\shortdownarrow}+1}^{\tau_1^{\shortdownarrow}+n} \left(\xi_k -c_l\right)} \geqslant -X_{\tau_1^{\shortdownarrow}} } \\
& \leqslant \p{\tau_1^{\shortdownarrow}<\infty} \p{\sup\limits_{n\geqslant 1} \skk{\sum\limits_{k=1}^n \left(\xi_k -c_l\right)} \geqslant 0} = p_{\shortdownarrow} p_{\shortuparrow} .
\end{split}
\end{equation}
Next, we consider the trajectory $\left\{X_n\right\}$ from $\tau_2^{\shortdownarrow}$. Up to and including the moment $\tau_3^{\shortdownarrow}$, the rambling jumps $\left\{X_n\right\}$ again coincide in distribution with $\xi_1 + c_r$ and by virtue of $ - X_{\tau_2^{\shortuparrow}} \leqslant 0$, we get
\begin{equation*}
\begin{split}
\p{\tau_3^{\shortdownarrow}<\infty} & = \p{\tau_2^{\shortuparrow}<\infty, \tau_3^{\shortdownarrow}<\infty}  = \p{\tau_2^{\shortuparrow}<\infty} \p{\inf\limits_{n \geqslant 1} \skk{\sum\limits_{k=\tau_2^{\shortuparrow}+1}^{\tau_2^{\shortuparrow}+n} \left(\xi_k +c_r\right)} < - X_{\tau_2^{\shortuparrow}} } \\ 
&\leq \p{\tau_2^{\shortuparrow}<\infty} \p{\inf\limits_{n \geqslant 1} \skk{\sum\limits_{k=1}^n \left(\xi_k +c_r\right)} < 0 } = p_{\shortdownarrow}^2 p_{\shortuparrow}.
\end{split}
\end{equation*}
Continuing this logic further, we obtain \eqref{eq_sharp_max_1}.

\ref{rk_sm_iii}. Under the current conditions, the densities for random variables $\xi_1 - c_l, \xi_1 + c_r$ have the following form
\begin{equation*}
\begin{split}
& f_{\xi_1 - c_l}(t)=f_{\xi_1}\left(t + c_l\right)=\left\{\begin{array}{lc}
r \beta\, e^{\beta \left(c_l + t\right)}, & t<-c_l \\
q \alpha \,e^{-\alpha (c_l+  t)}, & t \geq-c_l
\end{array}\right. \\
& f_{\xi_1 + c_r}(t)=f_{\xi_1}\left(t - c_r\right)=\left\{\begin{array}{lc}
r \beta \, e^{-\beta (c_r - t)}, & t < c_r \\
q \alpha \, e^{\alpha (c_r - t)}, & t \geqslant c_r
\end{array}\right. 
\end{split}
\end{equation*}
Thus, $\xi_1 - c_l$ has an exponential density on the positive semi-axis, and $\xi_1 + c_r$ has an exponential density on the negative semi-axis. It is known (\citet[Ch.12, Example 5.1]{Borovkov:20013} ) that under these conditions for the probabilities 
\begin{equation*}
 p_{\shortdownarrow}:= \p{\inf\limits_{n \geqslant 1} \skk{\sum\limits_{k=1}^n \left(\xi_k +c_r\right)} < 0}, \ p_{\shortuparrow}:=\p{\sup\limits_{n\geqslant 1} \skk{\sum\limits_{k=1}^n \left(\xi_k -c_l\right)} \geqslant 0}
\end{equation*}
the equation \eqref{eq_sharp_max_5} is true. 

It remains to show that in the double-exponential case the inequalities in~\eqref{eq_sharp_max_1} become equalities. This follows from the fact that, in the exponential case, the random variables $\tau_1^{\downarrow}$ and $X_{\tau_1^{\downarrow}}$ are independent, which immediately implies that equality holds in~\eqref{eq_loc_new_1} rather than an inequality. Remark~\ref{rk_sharp_max} is proved. 
\end{proof}  

\begin{proof}[Proof of Theorem~\ref{th_sharp_max}]
Let us prove the statement of the theorem for the exit through the right boundary. The statement for the left boundary follows from $\p{N^{\varepsilon} < \infty} = 1$ (see Lemma~\ref{bound_prob_lemma}). 

 Since the first exit through the right boundary can occur only in one of the intervals $\left[\tau_{2 k}^{\shortuparrow}(\varepsilon), \tau_{2k+1}^{\shortdownarrow}(\varepsilon)\right),$ then 
\begin{equation}\label{eq_sharp_max_8}
\begin{split}
\p{x_{N^{\varepsilon}}^{\varepsilon} > \delta} & = \sum_{k \geqslant 0} \p{\tau_{2 k}^{\shortuparrow}(\varepsilon) \leqslant N^{\varepsilon} < \tau_{2k+1}^{\shortdownarrow}(\varepsilon)} \\
& = \sum_{k \geqslant 0} \left( \p{\tau_{2 k}^{\shortuparrow}(\varepsilon) \leqslant N^{\varepsilon}}  - \p{ \tau_{2k+1}^{\shortdownarrow}(\varepsilon) \leqslant N^{\varepsilon}}\right).
\end{split}
\end{equation}
Let us show that for all $k \geq 1$ as $\varepsilon \downarrow 0$ we have
\begin{equation}\label{eq_sharp_max_7}
\p{\tau_{2 k}^{\shortuparrow}(\varepsilon) \leqslant N^{\varepsilon}} \ra \p{\tau_{2 k}^{\shortuparrow} < \infty}, \  \p{ \tau_{2k+1}^{\shortdownarrow}(\varepsilon) \leqslant N^{\varepsilon}} \ra \p{\tau_{2 k +1}^{\shortdownarrow} < \infty},
\end{equation}
and, the inequalities similar to \eqref{eq_sharp_max_1} are satisfied 
\begin{equation}\label{eq_sharp_max_6}
   \p{ \tau_{2 k-1}^{\shortdownarrow}(\varepsilon) \leqslant N_{\varepsilon}} \leqslant p_{\shortdownarrow}^k p_{\shortuparrow}^{k-1}, \quad \p{ \tau_{2 k}^{\shortuparrow}(\varepsilon) \leqslant N_{\varepsilon}} \leqslant (p_{\shortdownarrow} p_{\shortuparrow})^k.
\end{equation}
From \eqref{eq_sharp_max_8}--\eqref{eq_sharp_max_6} follows the statement of Theorem~\ref{th_sharp_max}.

To prove \eqref{eq_sharp_max_7} and \eqref{eq_sharp_max_6}, we note that the sequence (by $j$) SGD is equivalent to RRW in the sense that it changes the drift when crossing the boundary $-\frac{x_0^{\varepsilon}}{\varepsilon},$ 
if it starts from zero. Indeed, on the event $\skk{\tau_{1}^{\shortdownarrow}(\varepsilon) \leqslant N_{\varepsilon}}$ from \eqref{eq_sharp_max_2}  we have
\begin{equation*}
\tau_{1}^{\shortdownarrow}(\varepsilon)  =\inf \left\{j>0: x_0^{\varepsilon} + \varepsilon \sum_{i=0}^{n-1}\left(\xi_{i+1} + c_r \right)  <0\right\} = \inf \left\{j>0:  \sum_{i=0}^{n-1}\left(\xi_{i+1} + c_r\right) < - \frac{x_0^{\varepsilon}}{\varepsilon}\right\}
\end{equation*}
Where from $x_0^{\varepsilon} \geq 0$ we have
\begin{equation}\label{eq_sharp_max_10}
\tau_{1}^{\shortdownarrow}(\varepsilon)  \geq \inf \left\{j>0:  \sum_{i=0}^{n-1}\left(\xi_{i+1} + c_r\right) < 0\right\}  = \tau_{1}^{\shortdownarrow}.
\end{equation}
To prove \eqref{eq_sharp_max_7}, we show that for all $k \geq 0$
\begin{equation}\label{eq_sharp_max_11}
    \p{ \tau_{2k+1}^{\shortdownarrow}(\varepsilon) \leqslant N_{\varepsilon}, \tau_{2k+1}^{\shortdownarrow}(\varepsilon) \neq \tau_{2k+1}^{\shortdownarrow} }  \ra 0, \quad \p{ \tau_{2k}^{\shortuparrow}(\varepsilon) \leqslant N^{\varepsilon}, \tau_{2k}^{\shortuparrow}(\varepsilon) \neq \tau_{2k}^{\shortuparrow} }  \ra 0.
\end{equation}
For this purpose we note that by virtue of the conditions $x_{0}^{\varepsilon} \geq 0, x_{0}^{\varepsilon}  = o(\varepsilon) $ as $\varepsilon \downarrow 0$ we have
\begin{equation*}
\begin{split}
\p{ \tau_{1}^{\shortdownarrow}(\varepsilon) \leqslant N_{\varepsilon}, \tau_{1}^{\shortdownarrow}(\varepsilon) \neq \tau_{1}^{\shortdownarrow} }  & \leqslant  \p{ \tau_{1}^{\shortdownarrow}<  \tau_{1}^{\shortdownarrow}(\varepsilon) < \infty } \\  & \leqslant  \p{\tau_{1}^{\shortdownarrow}< \infty,  X_{\tau_{1}^{\shortdownarrow}} \in  \left[-\frac{x_0^{\varepsilon}}{\varepsilon}, 0\right)}  \ra 0.
\end{split}
\end{equation*}
Next, on a event $A_{\varepsilon}:=\skk{ \tau_{1}^{\shortdownarrow}(\varepsilon)  = \tau_{ 1}^{\shortdownarrow}, \tau_{2}^{\shortuparrow}(\varepsilon) \leqslant N^{\varepsilon}}$ we have $\tau_{2}^{\shortuparrow}(\varepsilon) \leqslant \tau_{2}^{\shortuparrow}.$ Hence, 
at $\varepsilon \downarrow 0$ we have
\begin{equation*}
\begin{split}
\p{ \tau_{2}^{\shortuparrow}(\varepsilon) \leqslant N^{\varepsilon}, \tau_{2}^{\shortuparrow}(\varepsilon) \neq \tau_{2}^{\shortuparrow} }  & \leqslant \p{ A_{\varepsilon}, \tau_{2}^{\shortuparrow}(\varepsilon) < \tau_{2}^{\shortuparrow} }  +  \p{ \tau_{1}^{\shortdownarrow}(\varepsilon) \neq \tau_{1}^{\shortdownarrow} } \\
& \leqslant  \p{ A_{\varepsilon}, X_{\tau_{2}^{\shortuparrow}(\varepsilon)} \in  \left[-\frac{x_0^{\varepsilon}}{\varepsilon}, 0\right)}  +  \p{ \tau_{1}^{\shortdownarrow}(\varepsilon) \neq \tau_{1}^{\shortdownarrow} }  \\ 
& \leqslant \p{S_{\tau_{1}^{\shortuparrow}} \in  \left[\eta - \frac{x_0^{\varepsilon}}{\varepsilon}, \eta \right)} + \p{ \tau_{1}^{\shortdownarrow}(\varepsilon) \neq \tau_{1}^{\shortdownarrow} } 
\ra 0,
\end{split}
\end{equation*}
we random variables $S_k := \sum_{n=1}^k (\xi_k - c_l), \tau_{1}^{\shortuparrow}: = \inf\skk{ k \geq 1 \,:\, S_k \geq 0}$ are independent from $ \eta := X_{\tau_1^{\shortdownarrow}} > \frac{x_0^{\varepsilon}}{\varepsilon}.$
Continuing this logic further, we obtain \eqref{eq_sharp_max_11}, 
 from which \eqref{eq_sharp_max_7} follows.

Now let's prove \eqref{eq_sharp_max_6}. From \eqref{eq_sharp_max_10}  and Lemma~\ref{bound_prob_lemma} we have
$$
\p{\tau_{1}^{\shortdownarrow}(\varepsilon) \leqslant N_{\varepsilon}} \leqslant  \p{ \tau_{1}^{\shortdownarrow} \leqslant N_{\varepsilon} } \leqslant \p{ \tau_{1}^{\shortdownarrow} < \infty } = p_{{\shortdownarrow}}.
$$
Similar to the reasoning of the proof of \ref{rk_sm_i} Remarks~\ref{rk_sharp_max}, we note that the value of the jump over level $0$ is negative$ x^{\varepsilon}_{\tau_{1}^{\shortdownarrow}}  < 0.$ Thus 
\begin{equation*}
\begin{split}
\p{\tau_{2}^{\shortuparrow}(\varepsilon) \leqslant N_{\varepsilon}}  & =    \p{\tau_{1}^{\shortdownarrow}(\varepsilon) \leqslant N_{\varepsilon}, \tau_{2}^{\shortuparrow}(\varepsilon) \leqslant N_{\varepsilon} } \\
& \leqslant \p{\tau_{1}^{\shortdownarrow}(\varepsilon) \leqslant N_{\varepsilon} ,  \sup\limits_{1 \leqslant n\leq N_{\varepsilon} } \skk{\sum\limits_{k=\tau_1^{\shortdownarrow}(\varepsilon)+1}^{\tau_1^{\shortdownarrow}(\varepsilon)+n} \left(\xi_k -c_l\right )} \geqslant - x^{\varepsilon}_{\tau_{1}^{\shortdownarrow}}} \\ 
& \leqslant \p{\tau_{1}^{\shortdownarrow} < \infty, \sup\limits_{1 \leqslant n\leq N_{\varepsilon} } \skk{\sum\limits_{k=\tau_1^{\shortdownarrow}(\varepsilon)+1}^{\tau_1^{\shortdownarrow}(\varepsilon)+n} \left(\xi_k -c_l\right )} \geqslant 0} \leqslant p_{{\shortdownarrow}} p_{{\shortuparrow}}.
\end{split}
\end{equation*}
Continuing this logic further, we obtain \eqref{eq_sharp_max_6}.  Theorem~\ref{th_sharp_max} is proved.

\end{proof}

\begin{proof}[Proof of Corollary~\ref{cl_sharp_max}]
\ref{cl_sm_i}. The statement \ref{cl_sm_i} follows from simple bounds on the series on the right-hand side of the equations \eqref{eq_sharp_max_4}, and \ref{rk_sm_i} from Remark~\ref{rk_sharp_max}
\begin{equation*}
\begin{split}
 \sum_{k \geqslant 0} \left( \p{  \tau_{2 k}^{\shortuparrow} < \infty} -\p{\tau_{2 k +1}^{\shortdownarrow} <\infty}\right) & \leqslant  1 - \p{\tau_{1}^{\shortdownarrow} <\infty} + \sum_{k \geqslant 1} \p{  \tau_{2 k}^{\shortuparrow} < \infty}  \\ 
& \leqslant 1 - p_{\shortdownarrow} +  \frac{p_{\shortdownarrow} p_{\shortuparrow} }{1- p_{\shortdownarrow} p_{\shortuparrow}},   \\
\sum_{k \geqslant 0} \left( \p{  \tau_{2 k+1}^{\shortdownarrow} < \infty} -\p{\tau_{2 k +2}^{\shortuparrow} <\infty}\right) & \leqslant  \sum_{k \geqslant 0} \p{  \tau_{2 k+1}^{\shortdownarrow} < \infty}  \leqslant \frac{p_{\shortdownarrow} }{1- p_{\shortdownarrow} p_{\shortuparrow}}.
\end{split}
\end{equation*}

\ref{cl_sm_ii}.
First, we prove that the output of the SGD sequence from the $\delta-$-neighborhood of the maximum satisfies the relation
\begin{equation}\label{eq_sharp_max_12}
    N_{\varepsilon} = o_{p}(n_\varepsilon)
\end{equation}
for all $n_{\varepsilon} \ra \infty \, :\,  \varepsilon n_{\varepsilon} \ra \infty$ as $\varepsilon \downarrow 0.$

Let us denote by $\tau_k(\varepsilon)$ the consecutive moments of zero crossing (regardless of top-down or bottom-up). That is $\tau_{2k}(\varepsilon):= \tau^{\shortuparrow}_{2k}(\varepsilon), \tau_{2k+1}(\varepsilon):= \tau^{\shortdownarrow}_{2k+1}(\varepsilon).$ 
Up to the moment $N^{\varepsilon}$ between the moments $\tau_k(\varepsilon)$, the derivative $f'(x_n^{\varepsilon})$ is constant and equals either $-c_r,$ or $c_l$. In the following, in order to shorten the calculations, we will denote the constant by a single symbol $c_*.$ For us it will be important only that $c_* \neq 0.$

For \eqref{eq_sharp_max_12}, we need only show that for all $t \in (0,1)$ it is $\p{N^{\varepsilon} > t n_{\varepsilon}} \ra 0.$
Consider the probability 
\begin{equation*}
    p(\varepsilon):= \p{ \left| \sum\limits_{i = 1}^{\left[t n_{\varepsilon}/2\right]} \left( \xi_i - c_* \right) \right| \leqslant \frac{2\delta}{\varepsilon} } = \p{ \frac{1}{\left[t n_{\varepsilon}/2\right]}\left| \sum\limits_{i = 1}^{\left[t n_{\varepsilon}/2\right]} \left( \xi_i - c_* \right) \right| \leqslant \frac{4\delta}{t n_{\varepsilon}\varepsilon+1} }. 
\end{equation*}
Since $\varepsilon n_{\varepsilon} \ra \infty, c_* \neq 0,$ then by LLN $p(\varepsilon) \ra 0$ as $\varepsilon \downarrow 0.$
Define 
$$
m_{\varepsilon}:=\begin{cases}
			\frac{1}{\sqrt{p(\varepsilon)}}, & \text{if $p(\varepsilon) \neq 0,$ }\\
            \frac{1}{\varepsilon}, & \text{otherwise.}
		 \end{cases}
$$
We have 
\begin{equation*}
\begin{split}
     \p{N^\varepsilon > t n_\varepsilon} & = \sum_{k=0}^{\infty} \p{N^\varepsilon > t n_\varepsilon, \ \tau_{k}(\varepsilon) \leqslant t n_{\varepsilon} <  \tau_{k+1}(\varepsilon) } \\ 
& \leqslant \sum_{k=0}^{m_{\varepsilon}-1} \p{N^\varepsilon > t n_\varepsilon, \ \tau_{k}(\varepsilon) \leqslant t n_{\varepsilon} <  \tau_{k+1}(\varepsilon) }  + \sum_{k=m_{\varepsilon}}^{\infty} \p{\tau_{k}(\varepsilon)  < N^{\varepsilon}}  =: \Sigma_1 + \Sigma_2. 
\end{split}
\end{equation*}
For $\Sigma_2$ from  \eqref{eq_sharp_max_6} as $\varepsilon \downarrow 0$ we have
\begin{equation*}
    \Sigma_2 \leqslant \sum_{k=m_{\varepsilon}}^{\infty}  (\max\skk{p_{\shortuparrow}, p_{\shortdownarrow}})^{2k} = \frac{(\max\skk{p_{\shortuparrow}, p_{\shortdownarrow}})^{2m_{\varepsilon}}}{1 - \max\skk{p_{\shortuparrow}, p_{\shortdownarrow}}} \ra 0.
\end{equation*}
For $\Sigma_1$, let's write out the sequence of inclusions. On the event $\skk{\tau_{k}(\varepsilon) \leqslant t n_{\varepsilon} <  \tau_{k+1}(\varepsilon)}$ we have 
\begin{equation*}
\begin{split}
     \skk{N^\varepsilon > t n_\varepsilon }  & = \skk{ \sup\limits_{0 \leqslant n \leqslant [t n_{\varepsilon}]} \left|x_n^{\varepsilon}\right| \leqslant \delta} 
     \subseteq \skk{ \sup\limits_{\tau_{k}(\varepsilon) \leqslant n \leqslant [t n_{\varepsilon}] } \left|x_{\tau_{k}(\varepsilon)}^{\varepsilon} +   \varepsilon \left( \sum\limits_{i = \tau_{k}(\varepsilon) + 1}^n \xi_i - c_*\right)\right| \leqslant \delta} \\ 
    & \subseteq \skk{ \sup\limits_{\tau_{k}(\varepsilon) < n \leqslant [t n_{\varepsilon}]} \left| \sum\limits_{i = \tau_{k}(\varepsilon) + 1}^n \xi_i - c_*\right| \leqslant \frac{2\delta}{\varepsilon}} \\
    & \subseteq \skk{ \sup\limits_{\tau_{k}(\varepsilon) < n \leqslant [t n_{\varepsilon}]} \left| \sum\limits_{i = \tau_{k}(\varepsilon) + 1}^n \xi_i - c_*\right| \leqslant \frac{2\delta}{\varepsilon}} \cap   \skk{ \tau_{k}(\varepsilon) \leqslant \frac{t n_{\varepsilon}}{2} }  \bigcup \skk{\tau_{k}(\varepsilon) \in \left( \frac{t n_{\varepsilon}}{2}, t n_{\varepsilon} \right] } \\
    & \subseteq \skk{  \left| \sum\limits_{i = \left[\frac{t n_{\varepsilon}}{2}\right] + 1}^{[t n_{\varepsilon}]} \xi_i - c_*\right| \leqslant \frac{2\delta}{\varepsilon}} \bigcup  \skk{\tau_{k}(\varepsilon) \in \left( \frac{t n_{\varepsilon}}{2}, t n_{\varepsilon} \right] } 
\end{split}
\end{equation*}
Thus, the inequality
\begin{equation*}
    \Sigma_1 \leqslant m_{\varepsilon} p(\varepsilon) + \sum_{k=0}^{m_{\varepsilon}-1} \p{\tau_{k}(\varepsilon) \in \left( \frac{t n_{\varepsilon}}{2}, t n_{\varepsilon} \right], N^\varepsilon > t n_\varepsilon } 
\end{equation*}
is satisfied. 
Using \eqref{eq_sharp_max_11}, we obtain for arbitrary fixed $K>0$ for all $k \leqslant K$
\begin{equation*}
    \p{\tau_{k}(\varepsilon) \in \left( \frac{t n_{\varepsilon}}{2}, t n_{\varepsilon} \right], N^\varepsilon > t n_\varepsilon } \ra 0.
\end{equation*}
By adding the inequalities \eqref{eq_sharp_max_6},  we have
\begin{equation*}
\sum_{k=0}^{m_{\varepsilon}-1} \p{\tau_{k}(\varepsilon) \in \left( \frac{t n_{\varepsilon}}{2}, t n_{\varepsilon} \right], N^\varepsilon > t n_\varepsilon }  \ra 0.
\end{equation*}
Form that $\Sigma_1 \ra 0.$ and hence $\p{N^\varepsilon > t n_\varepsilon} \ra 0.$ The relation \eqref{eq_sharp_max_12} is proved. 

Using \eqref{eq_sharp_max_12} and the strong Markov property we can show that \eqref{eq_sharp_max_13} is satisfied. Indeed, from Theorems~\ref{t.1}, \ref{t.2} we have 
\begin{equation*}
\begin{split}
       & \lim_{\varepsilon \downarrow 0} \p{ |x_{\lfloor n_\varepsilon \rfloor 
       }^{\varepsilon} - m_r | < \delta'\, |\, x_{N^{\varepsilon}}^{\varepsilon} < M_r -\delta } =  1,  \\ & \lim_{\varepsilon \downarrow 0} \p{ |x_{\lfloor n_\varepsilon \rfloor 
       }^{\varepsilon} - m_r | < \delta'\, |\, x_{N^{\varepsilon}}^{\varepsilon} > M_r +\delta } =  0. 
\end{split}
\end{equation*}
It remains to use Theorem~\ref{th_sharp_max}.
 \end{proof}

\end{document}